\documentclass[letterpaper, 10pt]{IEEEtran}
\usepackage{amsmath,graphicx,epsfig,color,amsfonts,subfigure}
\usepackage{version,xspace}

\usepackage[round, sort]{natbib}
\renewcommand{\cite}[1]{\citep{#1}}

\usepackage[vlined,ruled]{algorithm2e}

\graphicspath{{./img/}}

\def\dist{\mathcal{D}}
\def\expt{\mathbb{E}}
\def\real{\mathbb{R}}

\def\naturals{\mathbb{N}}

\newcommand{\until}[1]{\{1,\dots, #1\}}

\newcommand{\subscr}[2]{#1_{\textup{#2}}}
\newcommand{\supscr}[2]{#1^{\textup{#2}}}
\newcommand{\setdef}[2]{\{#1 \; | \; #2\}}
\newcommand{\seqdef}[2]{\{#1\}_{#2}}

\newcommand{\map}[3]{#1: #2 \rightarrow #3}
\newcommand{\union}{\operatorname{\cup}}

\newcommand{\subject}{\text{subject to}}

\newcommand{\minimize}{\text{minimize}}

\newcommand{\mc}{\mathcal}
\newcommand{\support}{\operatorname{supp}}
\newcommand{\degree}{\ensuremath{^\circ}}

\newcommand\oprocendsymbol{\hbox{$\square$}}
\newcommand\oprocend{\relax\ifmmode\else\unskip\hfill\fi\oprocendsymbol}

\renewcommand{\theenumi}{(\roman{enumi}}

\newcommand\bit[1]{\textit{\textbf{#1}}}

\newtheorem{theorem}{Theorem}
\newtheorem{lemma}[theorem]{Lemma}
\newtheorem{remark}{Remark}

\newtheorem{example}{Example}
\newtheorem{definition}[theorem]{Definition}

\newtheorem{conjecture}[theorem]{Conjecture}

\def \q{\boldsymbol{q}}
\def \p{\boldsymbol{p}}

\def \gav{\subscr{\delta}{avg}}

\def \Tone {\bar{T}_{\textup{one}}}
\def \Tbar {\bar{T}}
\def \Tsmall {\bar{T}^{m\textup{-smlst}}}

\def \etab {\bar{\eta}}
\def \a{\boldsymbol{a}}

\title{Stochastic Surveillance Strategies \\ for Spatial Quickest Detection}
\author{Vaibhav~Srivastava~\hspace{1in}~Fabio~Pasqualetti~\hspace{1in}~Francesco~Bullo
  \thanks{A preliminary version of this work~\citep{VS-FB:11m} was
    presented at IEEE Conference on Decision and Control and European
    Control Conference, 2011. In addition to the ideas
    in~\citep{VS-FB:11m}, this paper contains a rigorous analysis of
    the single vehicle surveillance, the multiple vehicle
    surveillance, extensive numerical illustrations, and a persistent
    surveillance experiment.}  \thanks{This work has been supported in
    part by AFOSR~MURI~Award-FA9550-07-1-0528, by NSF
    Award~CPS-1035917 and by ARO Award W911NF-11-1-0092.}
  \thanks{Vaibhav~Srivastava, Fabio~Pasqualetti and Francesco~Bullo
    are with Center for Control, Dynamical Systems, and Computation,
    University of California, Santa Barbara, Santa Barbara, CA 93106,
    USA, {\tt{\{vaibhav,fabiopas, bullo\} @engineering.ucsb.edu}}}}

\begin{document}

\maketitle

%%%%%%%%%%%%%%%%%%%%%%%%%%%%%%%%%%%%%%%%

\begin{abstract}
  We design persistent surveillance strategies for the quickest
  detection of anomalies taking place in an environment of interest.
  From a set of predefined regions in the environment, a team of
  autonomous vehicles collects noisy observations, which a control
  center processes.  The overall objective is to minimize detection
  delay while maintaining the false alarm rate below a desired
  threshold.  We present joint (i) anomaly detection algorithms for
  the control center and (ii) vehicle routing policies.  For the control
  center, we propose parallel cumulative sum (CUSUM) algorithms (one
  for each region) to detect anomalies from noisy observations.  For
  the vehicles, we propose a stochastic routing policy, in which the
  regions to be visited are chosen according to a probability vector.
  We study stationary routing policy (the probability vector is constant)
  as well as adaptive routing policies (the probability vector varies in
  time as a function of the likelihood of regional anomalies).
  In the context of stationary policies, we design a
  performance metric and minimize it to design an efficient stationary
  routing policy.
  % For stationary policies, we develop performance metrics and design
  % efficient routing policies for anomaly detection.
%%
  Our adaptive policy improves upon the stationary counterpart by
  adaptively increasing the selection probability of regions with high
  likelihood of anomaly. Finally, we show the effectiveness of the
  proposed algorithms through numerical simulations and a persistent
  surveillance experiment.
\end{abstract}

\begin{keywords}
vehicle routing, statistical decision making, quickest detection,
persistent surveillance, patrolling, security, motion planning.
\end{keywords}

\section{Introduction}
Recent years have witnessed a surge in the application of autonomous agents
in various activities such as surveillance and information collection. In
view of the recent Icelandic ash problem, the oil spill in the gulf of
Mexico, and recurring wild fires, surveillance strategies resulting in the
quickest detection of anomalies are of considerable importance. Due to
extreme sensor and modeling uncertainties in these situations, robust
anomaly detection methods need to be employed. Generally, a limited number
of vehicles are deployed to survey a large number of regions, and it is
fundamental that the vehicles collect the information that is most
effective to minimize the detection delay of anomalies.  In this paper we
design surveillance strategies that result in quick detection of anomalies.

% The extreme uncertainties in these situations result in quite noisy
% observations, and it becomes important to utilize robust methods to
% detect anomalies.  Generally, a limited number of robots are deployed
% to survey a large number of regions, and it becomes important that the
% robots collect the information that is most pertinent to the detection
% of an anomaly.

A reliable detection of anomalies can be achieved by collecting
observations sequentially until the evidence suggesting an anomaly
reaches a substantial level.  
Various sequential algorithms for the
detection of anomalies have been presented in~\cite{MB-IVN:93}. 
%The CUSUM
%algorithm is one of the celebrated test in this category: it minimizes
%the expected delay in detecting an anomaly for a given bound on the
%false alarm rate. 
Furthermore, it is known that a human being typically performs
well in detecting and identifying anomalies from
observations. Recent advances in cognitive
psychology~\cite{RB-EB-etal:06, RR-JG-MEN:03}, show that human
performance in decision making tasks is well modeled by sequential
statistical procedures such as the CUSUM algorithm.  
Inspired by the above human decision making models, in this work we adopt sequential statistical tests
for anomaly detection.

%The optimality
%property of the CUSUM algorithm and its connection 

%\fpmargin{we use vehicle and robot interchangeably}
%\noindent
\textbf{Our setup and approach.}  We consider an environment
comprising of potentially disjoint regions of interest, and we employ
a team of autonomous vehicles for the persistent surveillance of these
regions. In particular, the vehicles visit the regions, collect
information, and send it to a control center. We study a spatial
quickest detection problem with multiple vehicles, that is, the
simultaneous quickest detection of anomalies at spatially distributed
regions when the observations for anomaly detection are collected by
autonomous vehicles. For this problem, we let the control center run
parallel CUSUM algorithms (one for each region) with the collected
information. The control center then decides on the presence of
anomalies in the regions. Finally, we design vehicle routing policies
to collect observations at different regions. Our vehicle routing
policies aim to minimize the anomaly detection time at the control
center.

% We study a spatial quickest detection problem with multiple vehicles,
% that is, the simultaneous detection of anomalies at spatially
% distributed regions when the observations for anomaly detection are
% collected by UAVs. Our vehicles routing policies to collect
% observations are designed to minimize the anomaly detection time at
% the fusion center.

% Our setup also models the case where the information is processed by
% a human operator at the fusion center.

% We study the optimal routing policies for a UAV performing
% surveillance. We consider a UAV which surveys a set of regions,
% collects the information and sends it to a fusion center. The fusion
% center, runs parallel CUSUM algorithms (one for each region) with the
% information collected and decides on the presence of any anomaly. This
% setup also models the situations where the UAV surveys a region,
% collects evidence, and sends that to a fusion center where an operator
% processes it to detect any anomaly in any region.  For a given
% stochastic routing policy, we determine the expected time the CUSUM
% algorithms at different regions take to detect any anomaly.  We
% minimize the expected detection time over the policy space and thus,
% obtain the policy for quickest detection of any anomaly.

%\noindent
\textbf{Related work.} Vehicle routing policies have witnessed a
lot of attention in the robotics and controls literature.  A survey on
dynamic vehicle routing policies for servicing tasks is presented
in~\cite{fb-ef-mp-ks-sls:10k}. Recently, the routing for information
aggregation has been of particular interest.
\citet{DJK-etal:10} present a vehicle routing policy for
optimal localization of an acoustic source.  They consider a set of
spatially distributed sensors and optimize the trade-off between the
travel time required to collect a sensor observation and the
information contained in the observation. They characterize the
information in an observation by the volume of the Cramer-Rao ellipsoid
associated with an optimal estimator.
\citet{GAH-UM-GSS:11b} study routing for an AUV to
collect data from an underwater sensor network. They developed
approximation algorithms for variants of the traveling salesperson
problem to determine efficient policies that maximize the information
collected while minimizing the travel time.
\citet{VG-THC-BH-RMM:06} study the estimation in a linear
dynamical system with the observations collected by a set of mobile
sensors. They determine stochastic trajectories for mobile sensors
that minimize the error covariance of the Kalman filter estimate.
\citet{DZ-CC-TB:11} study the
estimation of environmental plumes with mobile sensors. They minimize the
uncertainty of the estimate of the ensemble Kalman filter to determine 
optimal trajectories for a swarm of mobile
sensors.  

There has been some interest in decision theoretic information
aggregation and vehicle routing as well.
\citet{DAC:95} poses the search problem as a dynamic
hypothesis test, and determines the optimal routing policy that
maximizes the probability of detection of a target.
\citet{THC-JWB:12} study the probabilistic search problem
in a decision theoretic framework. They minimize the search decision
time in a Bayesian setting.
Certain optimal information aggregation strategies for sequential
hypothesis testing have been developed in~\cite{VS-KP-FB:08p,
  VS-KP-FB:10l}.
\citet{GAH-UM-GSS:11a} study an active classification
problem in which an autonomous vehicle classifies an object based on
multiple views. They formulate the problem in an active Bayesian
learning framework and apply it to underwater detection. 
The persistent surveillance problem in this paper also concerns with decision-theoretic information aggregation and vehicle routing.
In contrast to the aforementioned works that focus on classification or search problems, our focus is on  quickest detection of anomalies.

The problem of surveillance has received considerable attention
recently. Preliminary results on this topic have been presented in
\cite{YC:04,YE-AS-GAK:08,DBK-RWB-RSH:08}. \citet{fp-af-fb:09v}
study the problem of optimal cooperative surveillance with multiple
agents. They optimize the time gap between any two visits to the same
region, and the time necessary to inform every agent about an event
occurred in the environment. \citet{SLS-DR:10} consider the
surveillance of multiple regions with changing features and determine
policies that minimize the maximum change in features between the
observations. A persistent monitoring task where vehicles move on a
given closed path has been considered
in~\cite{SLS-MS-DR:11,FP-JWD-FB:11h}, and a speed controller has been
designed to minimize the time lag between visits of regions.

Stochastic surveillance and pursuit-evasion problems have also fetched
significant attention. In an earlier work,~\citet{JPH-HJK-SSS:99} studied multi-agent probabilistic
pursuit evasion game with the policy that, at each instant, directs
pursuers to a location that maximizes the probability of finding an
evader at that instant.
\citet{JG-JB:05} formulate the surveillance problem as a
random walk on a hypergraph and parametrically vary the local
transition probabilities over time in order to achieve an accelerated
convergence to a desired steady state distribution.
\citet{TS-JW-SG:08} present  partitioning and routing 
strategies for surveillance of regions for different intruder models.
\citet{KS-DMS-MWS:09} present a stochastic surveillance problem in
centralized and decentralized frameworks. They use a Markov chain Monte
Carlo method and a message passing based auction algorithm to achieve the desired
surveillance criterion.  They also show that the deterministic strategies
fail to satisfy the surveillance criterion under general conditions.
In this paper, we focus on stochastic surveillance policies. In contrast to aforementioned works on stochastic surveillance that assume a surveillance criterion is known, this work concerns the design of the surveillance criterion. The policies designed in this paper direct a vehicle with high probability to a region with high probability of being anomalous, a feature akin to the heuristic  policy for the pursuer in~\cite{JPH-HJK-SSS:99}. 
On the other hand, with respect to~\cite{JPH-HJK-SSS:99}, our policy 
takes into account environmental factors, e.g., travel times and detection difficulty, and it satisfies an optimality criterion.

%\noindent
\textbf{Paper contributions.} The main contributions of this work are
fivefold. First, we formulate the stochastic surveillance problem for
spatial quickest detection of anomalies (Section \ref{sec:setup}). We
propose the ensemble CUSUM algorithm for a control center to detect
concurrent anomalies at different regions from collected observations
(Section \ref{sec:spatial-quickest-detection}). For the ensemble CUSUM
algorithm we characterize lower bounds for the expected detection
delay and for the average (expected) detection delay at each
region. Our bounds take into account the processing times for
collecting observations, the prior probability of anomalies at each
region, and the anomaly detection difficulty at each region.
%\fpmargin{Do we consider the travel time here?}

Second, for the case of stationary routing policies, we provide bounds
on the expected delay in detection of anomalies at each region
(Section \ref{sec:randomized-ensemble-cusum}). In particular, we take
into account both the processing times for collecting observations and
the travel times between regions. For the single vehicle case, we
explicitly characterize the expected number of observations necessary
to detect an anomaly at a region, and the corresponding expected
detection delay. For the multiple vehicles case, we characterize lower
bounds for the expected detection delay and the average detection
delay at the regions. As a complementary result, we show that the
expected detection delay for a single vehicle is, in general, a
non-convex function. However, we provide probabilistic guarantees that
it admits a unique global minimum.

Third, we design stationary vehicle routing policies to collect
observations from different regions (Section
\ref{sec:randomized-ensemble-cusum}). For the single vehicle case, we
design an efficient stationary policy by minimizing an upper bound for
the average detection delay at the regions. For the multiple vehicles
case, we first partition the regions among the vehicles, and then we
let each vehicle survey the assigned regions by using the routing
policy as in the single vehicle case. In both cases we characterize
the performance of our policies in terms of expected detection delay
and average (expected) detection delay.

Fourth, we describe our adaptive ensemble CUSUM algorithm, in which
the routing policy is adapted according to the learned likelihood of
anomalies in the regions (Section \ref{sec:adaptive-policy}). We
derive an analytic bound for the performance of our adaptive
policy. Finally, our numerical results show that our adaptive policy
outperforms the stationary counterpart.

Fifth and finally, we report the results of extensive numerical
simulations and a persistent surveillance experiment (Sections
\ref{sec:numerical} and \ref{sec:experimental-results}). Besides
confirming our theoretical findings, these practical results show that
our algorithm are robust against realistic noise models, and sensors
and motion uncertainties.

\section{Problem Setup}\label{sec:setup}
We consider the \emph{persistent surveillance} of a set of $n$
disjoint regions with a team of $m<n$ identical\footnote{The vehicle
  routing policies designed in this paper also work for non-identical
  vehicles. We make this assumption for the convenience of analysis.}
autonomous vehicles capable of sensing, communicating, and moving from
one region to another.  In persistent surveillance, the vehicles visit
the regions according to some routing policy, collect evidence (sensor
observation), and send it to a control center. The control center runs
an anomaly detection algorithm with the evidence collected by the
vehicles to determine the likelihood of an anomaly being present at
some region (the control center declares an anomaly if substantial
evidence is present). Finally, the control center utilizes the
likelihood of an anomaly at each region to determine a vehicle routing
policy. The objective of the control center is to detect an anomaly at
any region in minimum time subject to a desired bound on the expected
time between any two subsequent false alarms. Notice that the time
required to detect an anomaly depends on the anomaly detection
algorithm and the time vehicles take to travel the regions. Thus, the
control center needs to minimize the anomaly detection time jointly
over anomaly detection policies and vehicle routing policies. Our
problem setup is shown in Fig.~\ref{fig:setup}.

\begin{figure}
    \centering
    \includegraphics[width=.9\columnwidth]{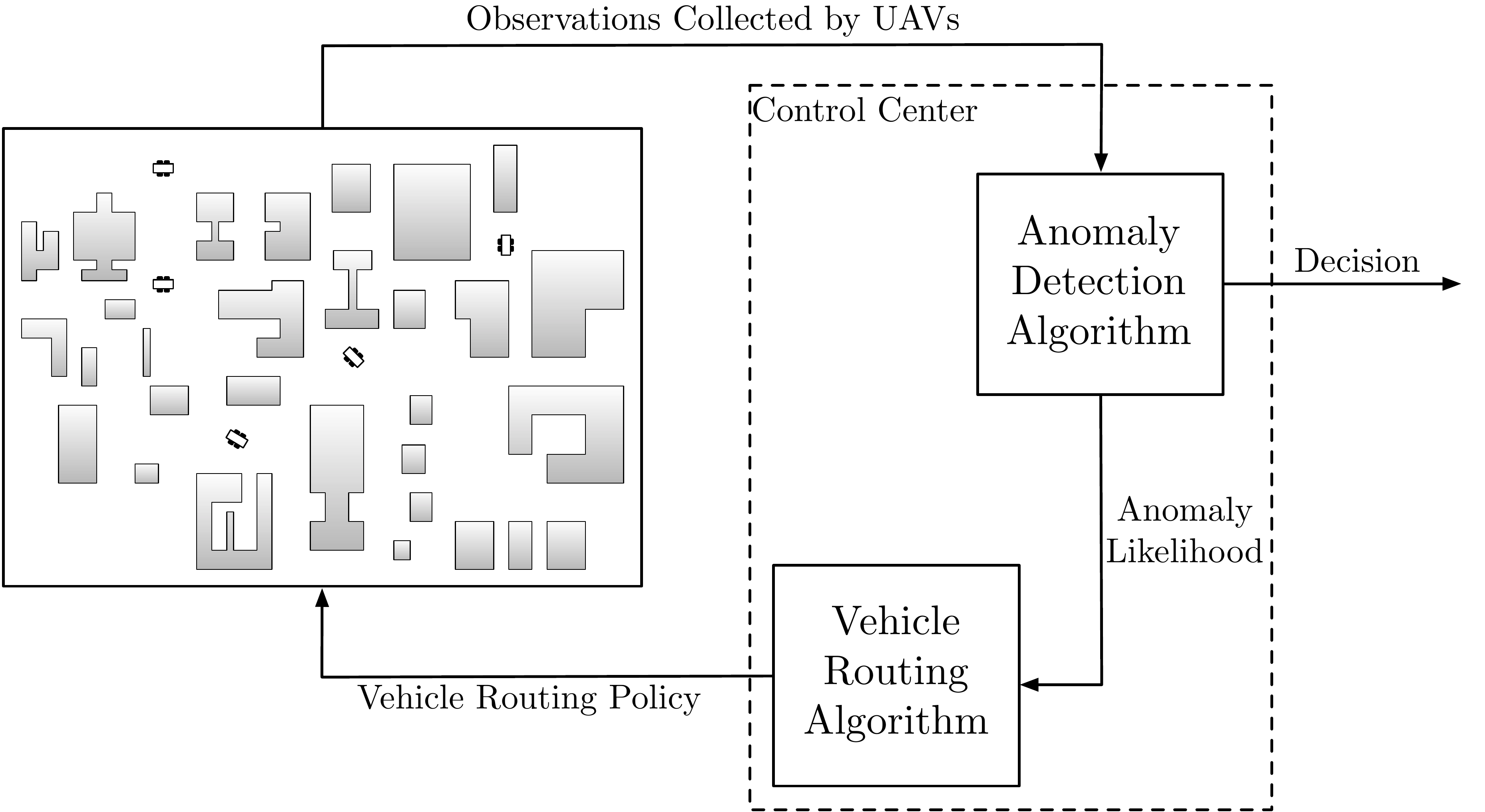}
    \caption{Persistent Surveillance Setup. A set of $n$ regions is
      surveyed by $m<n$ vehicles. Each vehicle visits the regions according
      to some policy and collects evidence from the visited region. The
      collected evidence is sent to an anomaly detection algorithm. The
      anomaly detection algorithm processes the collected evidence and
      decides on the presence of an anomaly. It also provides the
      likelihood of an anomaly being present, which in turn is used by the
      vehicle routing algorithm. The anomaly detection algorithm and
      vehicle routing algorithm constitute the control center, which
      can be implemented on-board of a vehicle.}
    \label{fig:setup}
\end{figure}

We adopt the standard motion planning notation in~\cite{SML:06}.  We
denote the $k$-th region by $\mathcal{R}_k, k\in \until{n}$, and
$r$-th vehicle by $\mathcal{U}_r, r \in \until{m}$. Let the likelihood
of an anomaly at region $\mc R_k$ be $\pi_k\in (0,1)$. We study the
persistent surveillance problem under the following assumptions.

Regarding the vehicles, we do not assume any specific dynamics and we
assume that:

\renewcommand{\theenumi}{(A\arabic{enumi}}
\begin{enumerate}
\item each vehicle takes time $d_{ij}$ to travel from region
  $\mathcal{R}_i$ to region $\mathcal{R}_j$, $i,j\in\until{n}$;
\item the sensors on each vehicle take a random time $T_k$ to collect
  an informative observation\footnote{An informative observation may require the
    acquisition of several observations from different locations at the
    same region. In this case the processing time equals the total
    time required to collect all these observations.} from region $\mathcal{R}_k, k\in
  \until{n}$.
  \newcounter{tmp}
  \setcounter{tmp}{\value{enumi}}
\end{enumerate}

Regarding the observations, we assume that:
\begin{enumerate}
\setcounter{enumi}{\value{tmp}}
\item the observation collected by a vehicle from region $\mathcal{R}_k$ is
  sampled from probability density functions 
  $\map{f^0_k}{\real}{\real_{\ge 0}}$ and
  $\map{f^1_k}{\real}{\real_{\ge 0}}$, respectively, in the presence
  and in the absence of anomalies;
\item for each $k\in \until{n}$, probability density functions $f^1_k$ and
  $f^0_k$ are non-identical with some non-zero probability, and the two distributions have the same support;
%  \fbmargin{maybe explain what we mean by demark? Is it a big assumption?}
\item conditioned on the presence or absence of anomalies, the
  observations in each region are mutually independent; and
\item observations in different regions are also mutually independent.
\end{enumerate}

Regarding the anomaly detection algorithm at the control center, we
employ the cumulative sum (CUSUM) algorithm (see below) for anomaly
detection at each region.  In particular, we run $n$ parallel CUSUM
algorithms (one for each region) and declare an anomaly being present
at a region as soon as a substantial evidence is present.  We refer to
such parallel CUSUM algorithms by {\it ensemble CUSUM algorithm}.

\renewcommand{\theenumi}{(\roman{enumi}}
\begin{remark}[\bit{Knowledge of distributions}]
  For the ease of presentation, we assume that the probability density
  functions in presence and absence of an anomaly are known. In
  general, only the probability density function in absence of any
  anomaly may be known, or both the probability density functions may
  be unknown.  In the first case, the CUSUM algorithm can be replaced
  by the weighted CUSUM algorithm or the Generalized Likelihood Ratio
  (GLR) algorithm~\cite{MB-IVN:93}, while in the second case, it
  can be replaced by the robust minimax quickest change detection
  algorithm~\cite{JU-VVV-SPM:11}.  The ideas presented in
  this paper extend to these cases in a straightforward way. A related
  example is in Section \ref{sec:numerical}. \oprocend
\end{remark}

\begin{remark}[\bit{Independence of observations}]
  For the ease of presentation, we assume that the observations
  collected from each region are independent conditioned on the
  presence and absence of anomalies. In general, the observations may
  be dependent and the dependence can be captured through an
  appropriate hidden Markov model.  If the observations can be modeled
  as a hidden Markov model, then the CUSUM like algorithm
  in~\cite{BC-PW:00} can be used instead of the standard CUSUM
  algorithm. The analysis presented in this paper holds in this case
  as well but in an asymptotic sense, i.e., in the limit when a large
  number of observations are needed for anomaly detection.

  We also assumed that the observations collected from different
  regions are mutually independent. Although the ideas in this paper
  also work when the observations at different regions are dependent,
  the performance can be improved with a slight modification in the
  procedure presented here (see Remark \ref{remark:dependence}). In
  this case the algorithm performance improves because each
  observation is now informative about more than one region.
  \oprocend
\end{remark}

Regarding the vehicle routing policy, we propose the \emph{randomized
  routing policy}, and the \emph{adaptive routing policy}. In the
randomized routing policy, each vehicle (i) selects a region from a
stationary distribution, (ii) visits that region, (iii) collects an
evidence, and (iv) transmits this evidence to the control center and
iterates this process endlessly. In the randomized routing policy, the
evidence collected by the vehicles is not utilized to modify their
routing policy. In other words, there is no feedback from the anomaly
detection algorithm to the vehicle routing algorithm. In the adaptive
routing policy, instead, the evidence collected by the vehicles is used
to modify the routing policy, and thus, the loop between the vehicle
routing algorithm and the anomaly detection algorithm is closed. The
adaptive routing policy follows the same steps as in the randomized
routing policy, with the exception that the distribution in step (i)
is no longer stationary and is adapted based on the collected evidence.

For brevity of notation, we will refer to the joint anomaly detection
and vehicle routing policy comprising of the ensemble CUSUM algorithm
and the randomized routing policy by \emph{randomized ensemble CUSUM
  algorithm.}  We will show that the randomized ensemble CUSUM
algorithm provides a solution that is within a factor of optimality.
Similarly, we refer to the joint anomaly detection and vehicle routing
policy comprising of the ensemble CUSUM algorithm and adaptive routing
policy by \emph{adaptive ensemble CUSUM algorithm.}  We will show that
adaptive ensemble CUSUM algorithm makes the vehicles visit anomalous
regions with high probability, and thus it improves upon the
performance of the randomized ensemble CUSUM algorithm. The following
standard definition~\cite{TMC-JAT:91} will be used in the
remaining sections.

\begin{definition}[\bit{Kullback-Leibler divergence}]
  Given two probability mass functions
  $\map{f_1}{\mathcal{S}}{\real_{\ge0}}$ and
  $\map{f_2}{\mathcal{S}}{\real_{\ge 0}}$, where $\mathcal{S}$ is some
  countable set, the Kullback-Leibler divergence
  $\map{\dist}{\mathcal{L}^1\times\mathcal{L}^1}{\real\union\{+\infty\}}$
  is defined by
 \begin{equation*}
   \dist(f_1,f_2)= \expt_{f_1}\bigg[\!\log\frac{f_1(X)}{f_2(X)}\bigg]=
   \sum_{x\in\support(f_1)} \!\!\! f_1(x) \log\frac{f_1(x)}{f_2(x)},
 \end{equation*}
 where $\mathcal{L}^1$ is the set of integrable functions,
 $\expt_{f_1}[\cdot]$ represents expected value with respect to $f_1$,
 $X$ is a random variable sampled from $f_1$, and $\support(f_1)$ is
 the support of $f_1$. \oprocend
\end{definition}
It is known that (i) $0\le\dist(f_1,f_2)\le+\infty$, (ii) the lower bound
is achieved if and only if $f_1=f_2$ almost everywhere, and (iii) the
upper bound is achieved if and only if the support of $f_2$ is a
strict subset of the support of
$f_1$. % Note that equivalent statements
% can be given for probability density functions.
Observe that Assumption (A4) on the observations is equivalent to
$\dist(f^1_k,f^0_k) >0$, for each $k\in \until{n}$.

We now introduce some notations that will be used throughout the
paper. 
We denote the probability simplex in $\real^n$
by $\Delta_{n-1}$, and the space of vehicle routing policies by
$\Omega$. For the processing time $T_k$, we let $\Tbar_k$ denote its
expected value. Consider $m$ realizations of the processing time
$T_k$, we denote the expected value of the minimum of these $m$
realized values by $\Tsmall_k$.  Note that
$\bar{T}^{1\textup{-smlst}}_k = \Tbar_k$. We also define $\Tbar_{\max}
= \max\setdef{\Tbar_k}{k\in\until{n}}$ and $\Tbar_{\min} = \min
\setdef{\Tbar_k}{k\in\until{n}}$. We denote the Kullback-Leibler
divergence between the probability density functions $f^1_k$ and
$f^0_k$ by $\dist_k$. Finally, $\dist_{\max} =
\max\setdef{\dist_k}{k\in\until{n}}$ and $\dist_{\min} =
\min\setdef{\dist_k}{k\in\until{n}}$.
For the convenience of the reader, we have enlisted the
notation in Table~\ref{tab:notations}.  \begin{table}[ht]
\centering
\caption{List of symbols \label{tab:notations}}
\resizebox{\linewidth}{!}{
\renewcommand{\arraystretch}{1.4}
\begin{tabular}{|c|l|}
\hline
$n$ & number of regions\\
$m$ & number of robots\\
$\mc R_k, k\in \until{n}$ & $k$-th region\\
$\mc U_r, r\in \until{m}$ & $r$-th vehicle\\
$\pi_k, k\in \until{n}$ & prior probability of anomaly at $\mc R_k$\\
$w_k$ & $\pi_k /(\sum_{j=1}^n \pi_j)$\\
$d_{ij}, i,j \in \until{n}$ & travel time between $\mc R_i$ and $\mc R_j$\\
$T_k, k\in \until{n}$ & processing time at $\mc R_k$\\
$\bar T_k, k\in \until{n}$ &expected processing time at $\mc R_k$\\
$\bar T_{\max}$ & $\max\setdef{\bar T_k}{k\in \until{n}}$\\
$\bar T_{\min}$ & $\min\setdef{\bar T_k}{k\in \until{n}}$\\
$\bar T_k^{m\text{-smlst}}$ & $\expt [\min\{T_k^{(1)}, \ldots, T_k^{(m)}\} ]$, where \\
& $T_k^{(1)}, \ldots, T_k^{(m)}$ are $m$ realizations of $T_k$\\
$\bar T_{\min}^{m\text{-smlst}}$  & $\min\setdef{\bar T_k^{m\text{-smlst}}}{k\in \until{n}}$\\ 
$\Xi$ &  set of sets of $m$ arbitrary regions\\
$\subscr{\bar T}{one}$ & $\min \setdef{\expt[\min\{t_1^{\xi}, \ldots,t_m^{\xi} \}]}{\xi \in \Xi}$,  where\\
& $t_i^{\xi}$'s are the processing times at regions in $\xi$\\
$f^0_k, k\in \until{n}$ & pdf in absence of anomaly at $\mc R_k$\\
$f^1_k, k\in \until{n}$ & pdf in presence of anomaly at $\mc R_k$\\
$\dist_k, k\in \until{n}$ & K-L divergence between $f^1_k$ and $f^0_k$\\
$\dist_{\max}$ & $\max\setdef{\dist_k}{k\in \until{n}}$\\
$\dist_{\min}$ & $\min\setdef{\dist_k}{k\in \until{n}}$\\
$\Omega$ & space of vehicle routing policies\\
$\map{N_k}{\Omega}{\naturals \union \{+\infty\}}$ & observations required for detection at $\mc R_k$\\
$\map{\delta_k}{\Omega}{\real_{>0}\union \{+\infty\}}$ & detection delay at $\mc R_k$\\
$\map{\subscr{\delta}{avg}}{\Omega}{\real_{>0}\union \{+\infty\}}$ & $\sum_{k=1}^n w_k \expt[\delta_k(\omega)]$\\
$\map{\subscr{\delta}{upper}}{\Omega}{\real_{>0}\union \{+\infty\}}$ & upper bound to $\subscr{\delta}{avg}$\\
 $ \delta_k^{m\textup{-min}}$& $\inf \setdef{\expt[\delta_k(\omega)]}{\omega \in \Omega}$\\
 $ \gav^{m\textup{-min}}$ & $\inf \setdef{\gav(\omega)}{\omega \in \Omega }$\\
 $\Lambda^j_{\tau}$ & CUSUM statistic at $\mc R_j$ at $\tau$-th iteration\\
 $\eta$ & CUSUM threshold\\
 $\etab$ & $e^{-\eta} +\eta -1$\\
 $\Delta_{n-1}$ & probability simplex in $\real^n$\\
 $\q \in \Delta_{n-1}$ & single vehicle randomized routing policy\\
$\q^* \in \Delta_{n-1}$ & optimal $\q$\\
$\q^\dag \in \Delta_{n-1}$ & efficient $\q$\\
$\vec \q_m \in \Delta_{n-1}^m$ & $m$ vehicle randomized routing policy\\
$\subscr{\vec \q}{part} \in \Delta_{n-1}^m$ & $\vec \q_m $ with region partitioning\\
$\boldsymbol{a} \in \Delta_{n-1}$ & single vehicle adaptive routing policy\\        
$\subscr{\boldsymbol{a}}{part} \in \Delta_{n-1}^m$ & $m$ vehicle adaptive routing policy \\
& \hfill  with region partitioning\\
\hline    
\end{tabular} 
}
\end{table}

\begin{remark}[\bit{Randomized routing policy}]
  The randomized routing policy samples regions to visit from a
  stationary distribution; this assumes that each region can be
  visited from another region in a single hop. While this is the case
  for aerial vehicles, it may not be true for ground vehicles. In the
  latter case, the motion from one region to another can be modeled as
  a Markov chain. The transition probabilities of this Markov chain
  can be designed to achieve a desired stationary distribution. This
  can optimally be done, for instance, by picking the fastest mixing
  Markov chain proposed in~\cite{SB-PD-LX:04} or heuristically by
  using the standard Metropolis-Hastings algorithm~\cite{LW:04}. Related examples are presented in Section
  \ref{sec:numerical} and \ref{sec:experimental-results}.  It should
  be noted that under the randomized routing policy, the desired
  stationary distribution of the Markov chain is fixed, and the Markov
  chain converges to this distribution exponentially.  Thus, the
  policy designed using Markov chain is arbitrarily close to the
  desired policy. However, in the case of adaptive routing policy, the
  desired stationary distribution keeps on changing, and the
  performance of the Markov chain based policy depends on rate of
  convergence of the Markov chain and the rate of change of desired
  stationary distribution.  \oprocend
\end{remark}

%\begin{remark}[\bit{Identical vehicles}] 
% \oprocend
%\end{remark}

\section{Spatial Quickest
  Detection}\label{sec:spatial-quickest-detection} In this section we
propose the ensemble CUSUM algorithm for the simultaneous quickest
detection of anomalies in spatially distributed regions. We start by
recalling the standard quickest change detection problem. Then we
describe and characterize the ensemble CUSUM algorithm.

\subsection{Quickest change detection}\label{subsec:cusum}
Consider a set of observations $\{y_1,y_2,\ldots\}$, where, for some
$\nu$, the observations $\{y_1,\ldots,y_{\nu-1}\}$ are i.i.d. with
probability density function $f^0$, and $\{y_\nu, y_{\nu +1},\ldots\}$
are i.i.d. with probability density function $f^1$.  The objective of
the quickest change detection is to detect the change in the
underlying distribution in minimum number of observations subject to a
desired lower bound on the number of samples between two false alarms.
Let $N \ge \nu$ be the observation at which the change is
detected. The non-Bayesian quickest detection problem~\cite{HVP-OH:08,
  DS:85}, is posed as
\begin{align}\label{eq:non-bayesian}
  \begin{split}
    \minimize & \quad  \sup_{\nu\ge 1}  \expt_{\nu} [N -\nu +1 | N\ge \nu]\\
    \subject & \quad \expt_{f^0}[N] \ge 1/ \gamma,
  \end{split}
\end{align}
where $\expt_\nu[\cdot]$ represents expected value with respect to the
observations distribution at iteration $\nu$ and
$\gamma>0$ is a small constant called {\it false alarm
  rate.}

An algorithmic solution to the minimization
problem~\eqref{eq:non-bayesian} is the cumulative sum (CUSUM)
algorithm~\cite{HVP-OH:08}, in which, at each iteration $\tau \in
\naturals$, (i) an observation $y_\tau$ is collected, (ii) the statistic
$\Lambda_{\tau}= \big ( \Lambda_{\tau-1}
+\log\frac{f^1_{k}(y_{\tau})}{f^0_{k}(y_{\tau})}\big)^+$ with
$\Lambda_0=0$ is computed, and (iii) a change is declared if
$\Lambda_{\tau}>\eta$. For a given threshold $\eta$, the false alarm
rate and the worst expected number of observations for CUSUM algorithm
are
\begin{align}\label{eq:cusum-delay}
  \expt_{f^0}(N) \approx \frac{ e^{\eta}-\eta - 1}{\dist(f^0,f^1)} \text{ and }
  \expt_{f^1}(N) \approx \frac{e^{-\eta}+\eta - 1}{\dist(f^1,f^0)}.
\end{align}
The approximations in equation~\eqref{eq:cusum-delay} are referred to as
the Wald's approximations~\cite{DS:85}, and are known to be accurate for
large values of the threshold $\eta$.  In the following, we assume that the
chosen threshold is large enough and the expressions in
equation~\eqref{eq:cusum-delay} are exact.  Let $u>0$ be the
uniform time duration between two iterations of the CUSUM algorithm. The
\emph {expected detection delay} $\delta$, i.e., the expected time required
to detect an anomaly after its appearance, satisfies $\expt_{f^1}[\delta]=
u \expt_{f^1}(N)$.

\subsection{Ensemble CUSUM algorithm}
We run $n$ parallel CUSUM algorithms (one for each region), and update
the CUSUM statistic for region $\mc R_k$ only if an observation is
received from region $\mc R_k$.  We refer to such parallel CUSUM
algorithms by {\it ensemble CUSUM algorithm}
(Algorithm~\ref{algo:ensemble-cusum}). Notice that an iteration of
this algorithm is initiated by the collection of an observation.

\IncMargin{.3em}
% \restylealgo{boxed}
% \linesnumbered 
\begin{algorithm}[t]
  {\footnotesize
   \SetKwInOut{Input}{Input}
   \SetKwInOut{Set}{Set}
   \SetKwInOut{Title}{Algorithm}
   \SetKwInOut{Require}{Require}
   \SetKwInOut{Output}{Output}
   \Input{threshold $\eta$, pdfs $f^0_k, f^1_k, k\in \until{n}$ \;}
   \Output{decision on presence of an anomaly \;}

   \medskip

   \nl at time $\tau$ receive observation $y_\tau$ for region $\mc R_k$\;

   \smallskip

   \nl update the CUSUM statistic at each region:
   
   \[
   \Lambda_{\tau}^{j}= 
   \begin{cases} 
     \Big ( \Lambda_{\tau-1}^{k}
     +\log\frac{f^1_{k}(y_{\tau})}{f^0_{k}(y_{\tau})}\Big)^+, &
     \text{if } j=k;\\
     \Lambda_{\tau-1}^{j}, & \text{if }
     j\in\until{n}\setminus\{k\};
   \end{cases}
   \]
%  \emph{\% decide only if the threshold is crossed}
   
   \smallskip
 
  \nl \lIf{$\Lambda_{\tau}^k >\eta$}{change detected at region $\mc
    R_k$ \;}
  \smallskip
  \nl \lElse{wait for next observations and iterate.}

  \medskip

  % \textbf{if} $\Lambda_{\tau}^k >\eta$, \quad \textbf{then}
  % declare change detected at region $k$%
  
  % wait for next sample and go to step~\algostep{1}

%  \nocaptionofalgo
    \caption{\textit{Ensemble CUSUM Algorithm}}
  \label{algo:ensemble-cusum}}
\end{algorithm} 
\DecMargin{.3em}

We are particularly interested in the performance of the ensemble CUSUM
algorithm when the observations are collected by autonomous vehicles. In
this case, the performance of the ensemble CUSUM algorithm is a function of
the vehicle routing policy.  For the ensemble CUSUM algorithm with
autonomous vehicles collecting observation, let the number of iterations
(collection of observations) required to detect an anomaly at region $\mc
R_k$ be $\map{N_k}{\Omega}{\naturals\union\{+\infty\}}$, and let the
detection delay, i.e., the time required to detect an anomaly, at region
$\mc R_k$ be $\map{\delta_k}{\Omega}{\real_{>0}\union\{+\infty\}}$, for
each $k\in \until{n}$, where $\Omega$ is the space of vehicle routing
policies.  We also define average detection delay as follows:
\begin{definition}[\bit{Average detection delay}]\label{def:detect-delay}
  For any vector of weights $(w_1,\ldots, w_n)\in \Delta_{n-1}$, define the
  average detection delay $\map{\gav}{\Omega}{\real_{>0} \union
    \{+\infty\}}$ for the ensemble CUSUM algorithm with autonomous vehicles
  collecting observations by
  \begin{equation}\label{eq:average-detection-delay-definition}
  \gav(\omega) = \sum_{k=1}^n w_k \expt[\delta_k (\omega)].
  \end{equation}
\end{definition}

For the ensemble CUSUM algorithm with $m$ vehicles collecting observation,
define $\delta_k^{m\textup{-min}}$ and $\gav^{m\textup{-min}}$ by
\begin{align*}
  \delta_k^{m\textup{-min}}&= \inf \setdef{\expt[\delta_k(\omega)]}{\omega \in \Omega}, \text{ and}\\
  \gav^{m\textup{-min}} &= \inf \setdef{\gav(\omega)}{\omega \in \Omega },
\end{align*}
respectively.  Note that $\delta_k^{m\textup{-min}}$ and
$\gav^{m\textup{-min}}$ are lower bounds for the expected detection
delay and average detection delay at region $\mc R_k$, respectively,
independently of the routing policy.  Let $\etab= e^{-\eta}+\eta - 1$.
We now state lower bounds on the performance of the ensemble CUSUM
algorithm with autonomous vehicles collecting observations.

\begin{lemma}[\bit{Global lower bound}] \label{lem:global-lower-bound}
  The following statements hold for the ensemble CUSUM algorithm with
  $m$ vehicles collecting information:
\begin{enumerate}
\item the lower bound $\delta_k^{m\textup{-min}}$ for the expected
  detection delay at region $\mc R_k$ satisfies
\[
\delta_k^{m\textup{-min}} \ge \frac{\etab\; \Tsmall_k}{m\dist_k};
\]
\item the lower bound $\gav^{m\textup{-min}}$ for the average detection delay
  satisfies
\[
\gav^{m\textup{-min}}  \ge \frac{\etab\; \Tsmall_{\min} }{m\dist_{\max}},
\]
where $\Tsmall_{\min} = \min \setdef{\Tsmall_k}{k\in \until{n}}$.
\end{enumerate}
\end{lemma}
\begin{proof}
  We start by establishing the first statement. We note that a lower bound
  on the expected detection delay at region $\mc R_k$ is obtained if all
  the vehicles always stay at region $\mc R_k$.  Since, each observation is
  collected from region $\mc R_k$, the number of iterations of the ensemble
  CUSUM algorithm required to detect an anomaly at region $\mc R_k$
  satisfies $\expt[N_k]=\etab/\dist_k$.
  Let $T_k^r(b)$ be realized value of the processing time of vehicle $\mc
  U_r$ at its $b$-th observation. It follows that
  $T^{m\text{-smlst}}_k(b)=\min\setdef{T_k^r(b)}{r\in \until{m}}$ is a
  lower bound on the processing time of each vehicle for its $b$-th
  observation. Further, $T^{m\text{-smlst}}_k(b)$ is identically
  distributed for each $b$ and
  $\expt[T^{m\text{-smlst}}_k(b)]=\Tsmall_k$. Consider a modified
  stochastic process where the realized processing time of each vehicle for
  its $b$th observation in $T^{m\text{-smlst}}_k(b)$. Indeed, such a
  stochastic process underestimates the time required to collect each
  observation and, hence, provides a lower bound to the expected detection
  delay.
  Therefore, the detection delay satisfies the following bound
  \[
  \delta_k(\omega) \ge \sum_{b=1}^{\lceil N_k/m\rceil} T^{m\text{-smlst}}_k(b), \text{ for each } \omega \in \Omega.
  \] 
  It follows from Wald's identity~\cite{SIR:99}, that
  \[
  \expt[\delta_k(\omega)] \ge \Tsmall_k \expt[\lceil N_k/m\rceil] \ge \Tsmall_k  \expt[N_k] /m.
  \]
  This proves the first statement. 
  
  The second statement follows from
  Definition~\ref{def:detect-delay} and the first statement.
\end{proof}
\begin{remark}[\bit{Dependence across regions}]\label{remark:dependence}
  We assumed that the observations collected from different regions
  are mutually independent. If the observations from different regions
  are dependent, then, at each iteration, instead of updating only one
  CUSUM statistic, the CUSUM statistic at each region should be
  updated with an appropriate marginal distribution.  \oprocend
\end{remark}

\section{Randomized Ensemble
  CUSUM Algorithm}\label{sec:randomized-ensemble-cusum}
We now study the persistent surveillance problem under randomized
ensemble CUSUM algorithm. First, we 
%characterize the expected
%detection delay for the randomized ensemble CUSUM algorithm. We 
derive an exact expression for the expected detection delay for the
randomized ensemble CUSUM algorithm with a single vehicle, and 
use the derived expressions to develop an efficient stationary policy for a single
vehicle.
 Second, we  develop a lower bound 
on the expected detection delay for the randomized ensemble CUSUM algorithm with multiple vehicles, and 
develop a generic partitioning policy that (i)
constructs a complete and disjoint $m$-partition of the regions, (ii)
allocates one partition each to a vehicle, and (iii) lets each vehicle
survey its assigned region with some single vehicle policy.  Finally,
we show that the partitioning policy where each vehicle implements the
efficient stationary policy in its regions is within a  factor
of an optimal policy.

\subsection{Analysis for single vehicle}
Consider the randomized ensemble CUSUM algorithm with a single vehicle.
Let $q_k\in {[0,1]}$ denote the probability to select region
$\mathcal{R}_k$, and let $\q=(q_1,\ldots,q_n) \in \Delta_{n-1}$.  Let
the threshold for the CUSUM algorithm at each region be uniform and
equal to $\eta >0$. We note that for the randomized
ensemble CUSUM algorithm with a single vehicle the space of vehicle
routing policies is $\Omega = \Delta_{n-1}$.

\begin{theorem}[\bit{Single vehicle randomized ensemble
    CUSUM}] \label{thm:ensemble-cusum} For the randomized ensemble
  CUSUM algorithm with a single vehicle and stationary routing policy
  $\q\in \Delta_{n-1}$, the following statements hold:
\begin{enumerate}
\item the number of observations $N_k(\q)$ required to detect a change
  at region $\mc R_k$ satisfies
\begin{align*}
  \expt_{f^1_k}[N_{k}(\q)] = \frac{\etab}{q_k \dist_k};
\end{align*}
\item the detection delay $\delta_{k}(\q)$ at region $\mc R_k$ satisfies
  \begin{align*}
    \expt_{f^1_k}[\delta_{k}(\q)] = \Big(\sum_{i=1}^n q_i\Tbar_i +
    \sum_{i=1}^n \sum_{j =1}^n q_iq_j d_{ij}\Big)
    \expt_{f^1_k}[N_{k}(\q)].
  \end{align*}
\end{enumerate}
\end{theorem}
\begin{proof}
  Let $\tau\in\until{N_{k}}$ be the iterations at which the vehicle
  collects and sends information about the regions, where $N_k$
  denotes the iteration at which an anomaly is detected at region $\mc
  R_k$. Let the log likelihood ratio at region $\mc R_k$ at iteration 
  $\tau$ be $\lambda_\tau^{k}$. We have
  \begin{align*}
    \lambda_\tau^{k}=
    \begin{cases}
      \log \frac{f^1_k(y_\tau)}{f^0_k (y_\tau)}, & \text{ with probability } q_{k},\\
      0, & \text{ with probability } 1-q_{k}.
    \end{cases}
  \end{align*}
  Therefore, conditioned on the presence of an anomaly,
  $\seqdef{\lambda_\tau^{k}}{\tau\in\naturals}$ are i.i.d., and
  \begin{align*}
    \expt_{f^1_k}[\lambda_\tau^{k}]=q_k \dist_k.
  \end{align*}
  The remaining proof of the first statement follows similar to the
  proof for CUSUM in \cite{DS:85}.
  
  To prove the second statement, note that the information aggregation
  time $T^{\text{agr}}$ comprises of the processing time and the
  travel time.  At an iteration the vehicle is at region $\mc R_i$ with
  probability $q_i$ and picks region $\mc R_j$ with probability
  $q_j$. Additionally, the vehicle travels between the two regions in $d_{ij}$
  units of time.  Thus, the average travel time at each iteration is
  \begin{align*}
    \expt[T_{\text{travel}}] = \sum_{i=1}^n \sum_{j =1}^n q_iq_j d_{ij}.
  \end{align*}
  Hence, the expected information aggregation time at each iteration is 
  \begin{align*}
    \expt[T^{\text{agr}}]=\expt[T_{\text{travel}}+T_{\text{process}}] =
    \sum_{i=1}^n \sum_{j =1}^n q_iq_j d_{ij} +
    {\sum_{i=1}^n  q_i\Tbar_i} .
  \end{align*}
  Let $\seqdef{{T_\tau^{\text{agr}}}}{\tau\in\until{N_{k}}}$, be
  the information aggregation times at each iteration. We have that
  $\delta_k =\sum_{\tau=1}^{N_{k}} T^{\text{agr}}_\tau$,
  and it follows from Wald's identity~\cite{SIR:99}, that
  \begin{align*}
    \expt[\delta_k]= \expt[T^{\text{agr}}] \expt[N_k].
  \end{align*}
  This completes the proof of the statement.
\end{proof}

\subsection{Design for single vehicle}
Our objective is to design a stationary policy that simultaneously
minimizes the detection delay at each region, that is, to design a
stationary policy that minimizes each term in $(\delta_1(\q),\ldots
,\delta_n(\q))$ simultaneously.  For this multiple-objective
optimization problem, we construct a single aggregate objective
function as the average detection delay. After incorporating the
expressions for the expected detection delays derived in
Theorem~\ref{thm:ensemble-cusum}, the average detection delay becomes
\begin{align}\label{eq:avg}
  \gav(\q)= \Big(\sum_{k=1}^n \frac{w_k \etab}{q_k \dist_k }\Big)
  \Big({{\sum_{i=1}^n q_iT_i}} + \sum_{i=1}^n \sum_{j=1}^n q_iq_j
  d_{ij}\Big),
\end{align}
where $w_k= \pi_k /(\sum_{i=1}^n \pi_i)$ is the weight on the expected
detection delay at region $\mc R_k$ and $\pi_k$ is the prior probability of
an anomaly being present at region $\mc R_k$. Our objective is to solve the
average detection delay minimization problem:
\begin{align}\label{eq:original-objective-function}
\begin{split}
\underset{\q\in\Delta_{n-1}}{\minimize} & \quad  \gav(\q).
\end{split}
\end{align}

\begin{figure}
    \centering
    \includegraphics[width=.5\columnwidth]{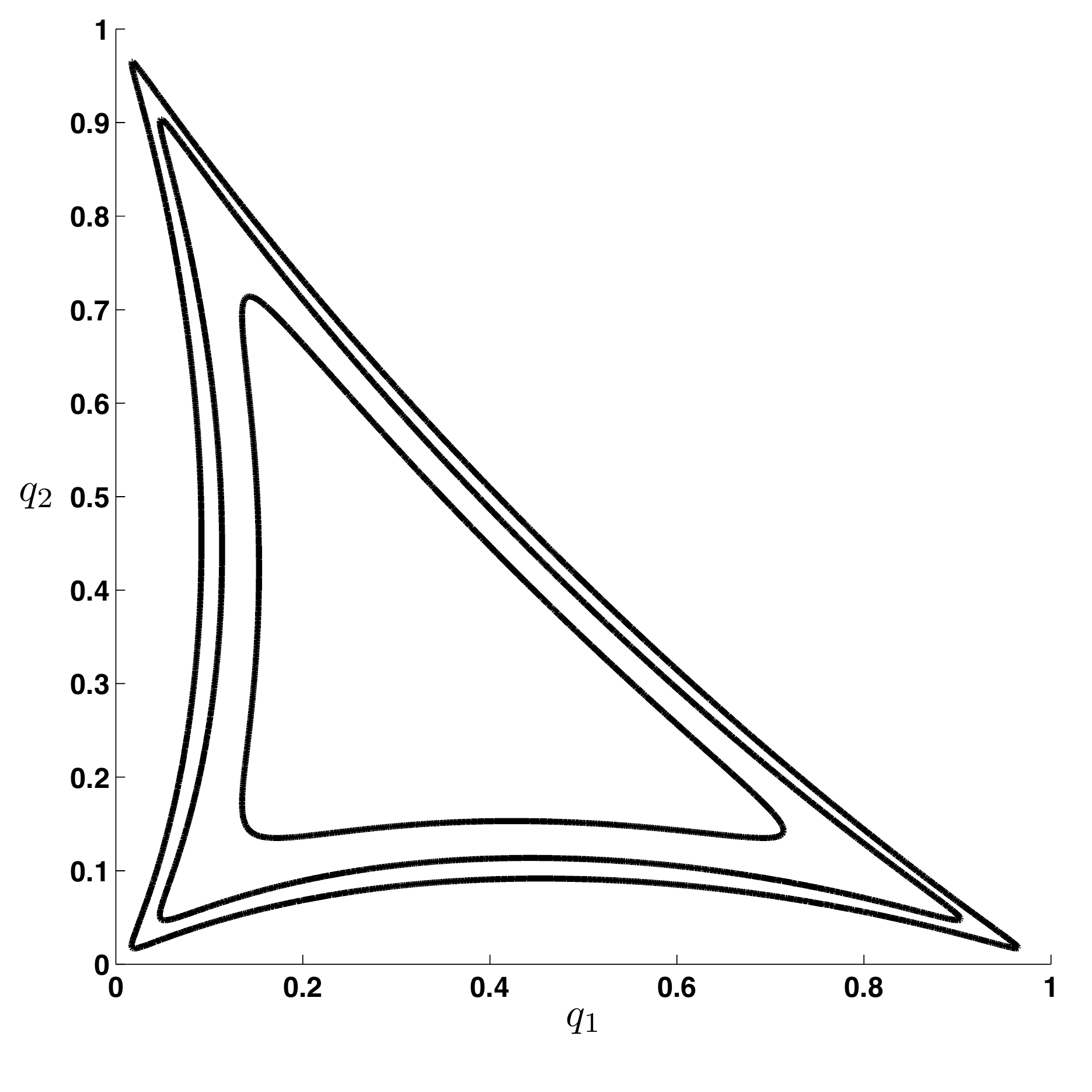}
    \caption{Level-sets of the objective function in
      problem~\eqref{eq:original-objective-function}. It can be seen
      that the level sets are not convex.}
    \label{fig:contours}
\end{figure}

In general, the objective function $\gav$ is non-convex. For
instance, let $n = 3$, and consider the level sets of $\gav$ on the
two dimensional probability simplex (Fig. \ref{fig:contours}).  It can
be seen that the level sets are non-convex, yet there exists a unique
critical point and it corresponds to a minimum. We now state the following
conjecture about the average detection delay:
\begin{conjecture}[\bit{Single vehicle optimal stationary policy}]\label{conj:unique}
  For the randomized ensemble CUSUM algorithm with a single vehicle, the
  average detection delay function $\gav$ has a unique critical point
  at which the minimum of $\gav$ is achieved. \oprocend.
\end{conjecture}

In the Appendix we provide probabilistic guarantees that, for a particular stochastic model of the parameters in
$\gav$, with at least
confidence level $99.99\%$ and probability at least $99\%$, the
optimization problem~\eqref{eq:original-objective-function} has a
unique critical point at which the minimum is achieved. Such a minimum 
can be computed via standard gradient-descent methods~\cite{SB-LV:04}.

We now construct an upper bound for the expected detection delay. We
will show that minimization of this upper bound yields a policy that
is within a factor of an optimal policy. From equation
\eqref{eq:avg}, we define the upper bound $\map{\subscr{\delta
  }{upper}}{\Delta_{n-1}}{\real_{>0}\union\{+\infty\}}$ as
\[
\gav(\q) \le \subscr{\delta}{upper}(\q ) = \Big(\sum_{k=1}^n
\frac{w_k\etab}{q_k \dist_k}\Big) (\Tbar_{\max}+d_{\max}),
\]
where $d_{\max}=\max\setdef{d_{ij}}{i,j\in\until{n}}$.

\begin{theorem}[\bit{Single vehicle efficient stationary
    policy}] \label{thm:ave-detection-delay}
  The following statements hold for the randomized ensemble CUSUM
  algorithm with single vehicle:
\begin{enumerate}
\item the upper bound on the expected detection delay satisfies
\begin{align*}
  \min_{\q \in \Delta_{n-1}} \subscr{\delta}{upper} (\q) &=
  \Big(\sum_{k=1}^n \sqrt{\frac{w_k}{\dist_k}}\Big)^2\etab
  (\Tbar_{\max} +d_{\max}),
\end{align*}
and the minimum is achieved at $\q^{\dag}$ defined by
\[
q^\dag_k = \frac{\sqrt{w_k/\dist_k}}{\sum_{j=1}^n \sqrt{w_j/\dist_j}} , \; k\in
\until{n};
\]
\item the average detection delay satisfies the following lower bound
\[
\gav(\q) \ge \Big(\sum_{k=1}^n \sqrt{\frac{w_k}{\dist_k}}\Big)^2 \etab\; \Tbar_{\min}, 
\]
for all $\q \in \Delta_{n-1}$;
\item the  stationary policy $\q^{\dag}$ is within a factor of
  optimal, that is
  \begin{align*}
  \frac{\gav(\q^\dag)}{\gav(\q^*)} &\le \frac{\Tbar_{\max}+d_{\max}}{\Tbar_{\min}}, \text{ and }\\
  \frac{\gav(\q^\dag)}{\gav^{1\textup{-min}}} & \le n \frac{\Tbar_{\max}+d_{\max}}{\Tbar_{\min}}\frac{\dist_{\max}}{\dist_{\min}},
  \end{align*}
where $\q^*$ is an optimal stationary policy;
\item the expected detection delay at region $\mc R_k$ under policy
  $\q^{\dag}$ satisfy
\[
\frac{\expt[\delta_k(\q^{\dag})]}{\delta^{1\textup{-min}}_k} \le \frac{(\Tbar_{\max}+d_{\max})}{\Tbar_k} \sqrt{\frac{n \dist_k} {w_k
    \dist_{\min}}}.
\]
\end{enumerate}
\end{theorem} 
\begin{proof}
  We start by establishing the first statement.  It follows from 
  the stationarity conditions on the Lagrangian that the
  minimizer $\q^\dag$ of $\subscr{\delta}{upper}$ satisfy $q_k^\dag \propto
  \sqrt{w_k \etab/\dist_k}$, for each $k\in\until{n}$. 
Incorporating this fact into $\sum_{k=1}^n q^\dag_k=1$ yields the expression for $q^\dag_k$.
  The expression for
  $\subscr{\delta}{upper}(\q^\dag)$ can be verified by substituting the 
  expression for $\q^\dag$ into $\subscr{\delta}{upper}$.

  To prove the second statement, we construct a lower bound
  $\map{\subscr{\delta}{lower}}{\Delta_{n-1}}{\real_{>0}\union
    \{+\infty\}}$ to the average detection delay $\gav$ defined by
  $\subscr{\delta}{lower}(\q) = \sum_{k=1}^n w_k \etab \bar T_{\min}/\dist_k q_{k}$.  
  It can be verified that $\subscr{\delta}{lower}$ also achieves its minimum at $\q^\dag$, and
\begin{align*}
  \subscr{\delta}{lower}(\q^\dag) &= \Big(\sum_{k=1}^n \sqrt{\frac{w_k}{\dist_k}}\Big)^2 \etab\; \Tbar_{\min}.
\end{align*}
We note that
\[
\subscr{\delta}{lower}(\q^\dag) \le \subscr{\delta}{lower}(\q^*) \le \gav(\q^*)
\le \gav(\q), \forall \q\in \Delta_{n-1}.
\]
Thus, the second statement follows.

To prove the first part of the third statement, we note that
\[
\subscr{\delta}{lower}(\q^\dag) \le \gav(\q^*) \le \gav(\q^\dag) \le
\subscr{\delta}{upper}(\q^\dag).
\]
Therefore, the policy $\q^\dag$ is within $\subscr{\delta}{upper}(\q^\dag)
/ \subscr{\delta}{lower}(\q^\dag) =(T_{\max}+d_{\max})/T_{\min}$ factor of
optimal stationary policy.

To prove the second part of the third statement, we note
\begin{align*}
\frac{\gav(\q^\dag)}{\gav^{1\textup{-min}}} &\le  \frac{\dist_{\max}(\Tbar_{\max}+d_{\max})}{\dist_{\min}\Tbar_{\min}}(\sqrt{w_1}+\ldots +\sqrt{w_n})^2\\
& \le   n \frac{\Tbar_{\max}+d_{\max}}{\Tbar_{\min}}\frac{\dist_{\max}}{\dist_{\min}},
\end{align*}
where the last inequality follows from the fact: $\max\setdef{\sqrt{w_1}+\ldots +\sqrt{w_n}}{w_1+\ldots+ w_n=1}=\sqrt{n}$.

To establish the last statement, we note that 
\begin{align*}
\frac{\expt[\delta_k(\q^{\dag})]}{\delta^{1\textup{-min}}_k} &\le \frac{ (\Tbar_{\max}+ d_{\max})}{q_k^\dag \Tbar_k}\\
&\le \frac{(\Tbar_{\max}+d_{\max})}{\Tbar_k} \sqrt{\frac{\dist_k} {w_k \dist_{\min}}}(\sqrt{w_1}+\ldots +\sqrt{w_n}) \\
&\le \frac{(\Tbar_{\max}+d_{\max})}{\Tbar_k} \sqrt{\frac{n \dist_k} {w_k \dist_{\min}}}.
\end{align*}
This concludes the proof of the theorem.
\end{proof}

In the following, we would refer to $\q^\dag$ as the \emph {single vehicle efficient stationary policy}.

\begin{remark}[\bit{Efficient stationary policy}]
  As opposed to the average detection delay $\gav$, the upper bound
  $\subscr{\delta}{upper}$ does not depend upon any travel time
  information. Then, our efficient policy does not take this information
  into account, and it may not be optimal. Instead, an optimal policy
  allocates higher visiting probabilities to regions that are located more
  centrally in the environment.  We resort to the efficient policy because
  (i) if the problem~\eqref{eq:original-objective-function} does not admit
  a unique minimum, then the optimal policy can not be computed
  efficiently; and (ii) the efficient policy has an intuitive, tractable,
  and closed form expression.  \oprocend
\end{remark}

\subsection{Analysis for multiple vehicles}
We now consider the randomized ensemble CUSUM with $m>1$ vehicles.  In this
setting the vehicles operate in an asynchronous fashion. This
asynchronicity, which did not occur in the single vehicle case, is due to
(i) different travel times between two different pair of regions, and (ii)
different realized value of processing time at each iteration.  Such an
asynchronous operation makes the time durations between two subsequent
iterations non-identically distributed and makes it difficult to obtain
closed form expressions for the expected detection delay at each region.

Motivated by the above discussion, we determine a lower bound on the
expected detection delay for the randomized ensemble CUSUM algorithm
with multiple vehicles. % Before we determine the lower bound on
% performance, we introduce some notation.
Let $\q^r=(q^r_1,\ldots,q^r_n) \in \Delta_{n-1}$ denote the
\emph{stationary policy} for vehicle $\mc U_r$, i.e., the vector of
probabilities of selecting different regions for vehicle $\mc U_r$,
and let $ \vec{\q}_m=(\q^1, \ldots, \q^m) \in \Delta_{n-1}^m$.  We
note that for the randomized ensemble CUSUM algorithm with $m$
vehicles the space of vehicle routing policies is $\Omega =
\Delta_{n-1}^m$. We construct a lower bound on the processing times at
different regions for different vehicles in the following way. Let
$\Xi$ be the set of all the sets with cardinality $m$ in which each entry 
is an arbitrarily chosen region; equivalently, $\Xi =\{\mc R_1,\ldots, \mc R_n\}^m$.
Let a realization of the
processing times at the regions in a set $\xi \in \Xi$ be $t_1^{\xi},
\ldots, t_m^{\xi}$.  We now define a lower bound $\Tone$ to the
expected value of the minimum of the processing times at $m$ arbitrary
regions as $\Tone = \min \setdef{\expt[\min\{t_1^{\xi}, \ldots,
  t_m^{\xi}\}]}{\xi\in \Xi}$.
%
%and define
%$T_{\textup{one}} =\min\{t_1, \ldots, t_m\}$. Note that
%$T_{\textup{one}}$ is a random variable and is sampled from the same
%distribution each time this process is repeated. Let $\Tone$ be the
%expected value of $T_{\textup{one}}$.

\begin{theorem}[\textit{\textbf{Multi-vehicle randomized ensemble
      CUSUM}}]\label{thm:stationary-lower-bound}
  For the randomized ensemble CUSUM algorithm with $m$ vehicles and
  stationary region selection policies $\q^r$, $r\in\until{m}$, the
  detection delay $\delta^k$ at region $\mc R_k$ satisfies:
\begin{align*}
  \expt_{f^1_k}[\delta_{k}(\vec{\q}_m)] \ge \frac{\etab\;
    \Tone}{\sum_{r=1}^m q_k^r \dist_k}.
\end{align*}
\end{theorem}
\begin{proof}
  We construct a modified stochastic process to determine a lower
  bound on the expected detection delay. For the randomized ensemble
  CUSUM algorithm with multiple vehicles, let $t_r^b$ be the the
  processing time for the vehicle $\mc U_r$ during its $b$-th visit to
  any region.  We assume that the sampling time for each vehicle at
  its $b$-th visit in the modified process is $\min\{t_1^b,\ldots,
  t_m^b\}$.  Therefore, the sampling time for the modified process is
  the same at each region. Further, it is identically distributed for
  each visit and has expected value greater than or equal to $\Tone$.  We further assume that
  the distances between the regions are zero. Such a process
  underestimates the processing and travel time required to collect
  each observation in the randomized ensemble CUSUM algorithm. Hence,
  the expected detection delay for this process provides a lower bound
  to the expected detection delay for randomized ensemble CUSUM
  algorithm.  Further, for this process the vehicles operate
  synchronously and the expected value of the likelihood ratio at
  region $k$ at each iteration is $\sum_{r=1}^m q^r_k
  \dist(f^1_k,f^0_k)$.  The remainder of the proof follows similar to
  the proof for single vehicle case in
  Theorem~\ref{thm:ensemble-cusum}.
\end{proof}

\subsection{Design for multiple vehicles}
We now design an efficient stationary policy for randomized ensemble CUSUM
algorithm with multiple vehicles. We propose an algorithm that partitions
the set of regions into $m$ subsets, allocates one vehicle to each subset,
and implements our single vehicle efficient stationary policy in each
subset. This procedure is formally defined in
Algorithm~\ref{algo:partitioning}.

\IncMargin{.3em}
% \restylealgo{boxed}
% \linesnumbered 
\begin{algorithm}[t]
  {\footnotesize
   \SetKwInOut{Input}{Input}
   \SetKwInOut{Set}{Set}
   \SetKwInOut{Title}{Algorithm}
   \SetKwInOut{Require}{Require}
   \SetKwInOut{Output}{Output}
   \Input{vehicles $\{\mc U_1, \ldots, \mc U_m\}$, regions $\mc R =\{\mc R_1, \ldots, \mc R_n\}$, \\ \hfill a single vehicle routing policy\;}
   \Require{$n>m$ \;}
   \Output{a $m$-partition of the regions \;}

   \medskip

   \nl  partition $\mc R$ into $m$ arbitrary subsets $\seqdef{\mc S^r}{r\in\until{m}}$ \\
   \hfill with cardinalities $n_r \le \lceil n/m \rceil, r\in\until{m}$ \;

   \smallskip
   
   \nl allocate vehicle $\mc U_r$ to subset $\mc S^r$, for each $r\in \until{m}$\;
   
   \smallskip
   
   \nl implement the single vehicle efficient stationary policy in each
   subset.

  \medskip

  % \textbf{if} $\Lambda_{\tau}^k >\eta$, \quad \textbf{then}
  % declare change detected at region $k$%
  
  % wait for next sample and go to step~\algostep{1}

%  \nocaptionofalgo
    \caption{\textit{Partitioning Algorithm}}
  \label{algo:partitioning}}
\end{algorithm} 
\DecMargin{.3em}

Let the subset of regions allocated to vehicle $\mc U_r$ be $\mc S^r, r\in \until{m}$.
We will denote the elements of subset $\mc S^r$ by
$\mc S^r_i , i\in \until{n_r}$. Let $\subscr{\vec{\q}^{\dag}}{part}\in
\Delta_{n-1}^m$ be a stationary routing policy under the partitioning
algorithm that implements single vehicle efficient stationary policy in
each partition. 
We define the weights in equation~\eqref{eq:average-detection-delay-definition} by $w_k=\pi^1_k/\sum_{j=1}^n \pi^1_j$,
where $\pi^1_k$ is the prior probability of anomaly at region $\mc R_k$.
Let $w_{\min}=\min\{w_1,\ldots, w_n\}$ and $w_{\max}=\max\{w_1,\ldots, w_n\}$.  
We now analyze the performance of the partitioning algorithm and show that it is
within a  factor of optimal.

\begin{theorem}[\bit{Performance of the partitioning
    policy}] \label{thm:multi-avg-detection-delay}
  For the partitioning algorithm with $m$ vehicles and $n$ regions that
  implements the single vehicle efficient stationary policy in each
  partition, the following statements hold:
\begin{enumerate}
\item the average detection delay under partitioning policy satisfies
  the following upper bound
\[
\gav({\subscr{\vec{\q}^{\dag}}{part}}) \le m
\left\lceil\frac{n}{m}\right\rceil^2 \frac{w_{\max} \etab
  (\Tbar_{\max} +d_{\max})}{\dist_{\min}};
\]
\item the average detection delay satisfies the following lower bound
\[
\gav(\vec{\q}_m) \ge \Big( \sum_{k=1}^n
\sqrt{\frac{w_k}{\dist_k}}\Big)^2 \frac{ \etab \Tone}{m} ,
\]
for any $\vec{\q}_m\in \Delta_{n-1}^m$;
\item the stationary policy $\subscr{\vec{\q}^{\dag}}{part}$ is within
  a factor of optimal, and
\begin{align*}
  \frac{\gav(\subscr{\vec{\q}^{\dag}}{part})}{\gav(\vec{\q}_m^*)} &\le  \frac{4w_{\max}}{w_{\min}}\frac{(\Tbar_{\max} +d_{\max})}{\Tone} \frac{\dist_{\max}}{\dist_{\min}}, \text{ and }\\
  \frac{\gav(\subscr{\vec{\q}^{\dag}}{part})}{\gav^{m\textup{-min}}} & \le
  m^2 \Big\lceil \frac{n}{m} \Big \rceil \frac{(\Tbar_{\max}
    +d_{\max})}{\Tsmall_{\min}}\frac{\dist_{\max}}{\dist_{\min}},
\end{align*}
where $\vec{\q}_m^*$ is optimal stationary policy;
\item the expected detection delay at region $\mc R_k$ under the
  stationary policy $\subscr{\vec{\q}^{\dag}}{part}$ satisfies
\begin{align*}
  \frac{\expt[\delta_k(\subscr{\vec{\q}^{\dag}}{part})]}{\delta^{m\textup{-min}}_k}
  \le \frac{m (\Tbar_{\max}+d_{\max})}{\Tsmall_k } \sqrt{ \Big\lceil
    \frac{n}{m}\Big \rceil \frac{\dist_k}{ w_k \dist_{\min}}}.
\end{align*}
\end{enumerate}
\end{theorem}
\begin{proof}
  We start by establishing the first statement. We note that under the
  partitioning policy, the maximum number of regions a vehicle serves is
  $\lceil n/m \rceil$. It follows from
  Theorem~\ref{thm:ave-detection-delay} that for vehicle $\mc U_r$ and
  the associated partition $\mathcal{S}^r$, the average detection
  delay is upper bounded by
  \begin{align*}
    \gav(\subscr{\q^r}{part}) &\le \Big(\sum_{i=1}^{n_r}\sqrt{\frac{w_i}{\dist_i}}\Big)^2 \etab  (\Tbar_{\max} +d_{\max})\\
    &\le \Big\lceil \frac{n}{m}\Big \rceil ^2 \frac{\etab
      w_{\max}(\Tbar_{\max}+d_{\max})}{\dist_{\min}}.
  \end{align*}
  Therefore, the overall average detection delay satisfies
  $\gav({\subscr{\vec{\q}^{\dag}}{part}}) \le m
  \gav(\subscr{\q^r}{part}) $.  This establishes the first statement.

  To prove the second statement, we utilize the lower bounds obtained
  in Theorem~\ref{thm:stationary-lower-bound} and construct a lower
  bound to the average detection delay
  $\map{\subscr{\delta}{lower}^m}{\Delta_{n-1}^m}{\real_{>0}\union
    \{+\infty\}}$ defined by $\subscr{\delta}{lower}^m(\vec{\q}_m)=
  \sum_{k=1}^n ({v_k \Tone}/{\sum_{r=1}^m q_k^r }).$ It can be
  verified that
 \[
 \min_{\vec{\q}_m\in \Delta_{n-1}^m} \subscr{\delta}{lower}^m
 (\vec{\q}_m)= \Big( \sum_{k=1}^n \sqrt{\frac{w_k}{\dist_k}}\Big)^2
 \frac{ \etab \Tone}{m}.
 \]

  We now establish the first part of the third statement. Note that
\begin{align*}
 \frac{\gav(\subscr{\vec{\q}^{\dag}}{part})}{\gav(\vec{\q}_m^*)} 
 & \le   \frac{\left\lceil n/m \right\rceil^2 } {(n/m)^2}\frac{w_{\max}}{w_{\min}}\frac{(\Tbar_{\max} +d_{\max})}{\Tone} \frac{\dist_{\max}}{\dist_{\min}}  \\ 
 &\le  \frac{4w_{\max}}{w_{\min}}\frac{(\Tbar_{\max} +d_{\max})}{\Tone} \frac{\dist_{\max}}{\dist_{\min}},
\end{align*}  
where the last inequality follows from the fact that $(\left\lceil n/m
\right\rceil) /(n/m) \le 2$.

The remainder of the proof follows similar to the proof of  Theorem~\ref{thm:ave-detection-delay}.
\end{proof}

\section{Adaptive ensemble CUSUM Algorithm}\label{sec:adaptive-policy}
The stationary vehicle routing policy does not utilize the real-time
information regarding the likelihood of anomalies at the regions. We
now develop an adaptive policy that incorporates the anomaly
likelihood information provided by the anomaly detection algorithm.
We consider the CUSUM statistic at a region as a measure of the
likelihood of an anomaly at that region, and utilize it at each
iteration to design new prior probability of an anomaly for each
region.  At each iteration, we adapt the efficient stationary policy
using this new prior probability.  This procedure results in higher
probability of visiting an anomalous region and, consequently, it
improves the performance of our efficient stationary policy. In
Section~\ref{sec:numerical} we provide numerical evidence showing that
the adaptive ensemble CUSUM algorithm improves the performance of
randomized ensemble CUSUM algorithm.

\IncMargin{.3em}
% \restylealgo{boxed}
% \linesnumbered 
\begin{algorithm}[t]
  {\footnotesize
   \SetKwInOut{Input}{Input}
   \SetKwInOut{Set}{Set}
   \SetKwInOut{Title}{Algorithm}
   \SetKwInOut{Require}{Require}
   \SetKwInOut{Output}{Output}

   \Input{parameters $\eta$, $\dist_k$, pdfs $f^0_k, f^1_k$, for each $k\in \until{n}$  \;}
   \smallskip
%   \Require{define assumptions \;}
   \Output{decision on anomaly at each region \;}

   \medskip
   
       \nl set $\Lambda^j_0=0$, for all $j\in \until{n}$, and $\tau=1$\;
               \smallskip

   \While{\textup{true}}{ %$|\mathcal{R}|>0$

		\nl set new prior $\pi^1_k= e^{\Lambda^k_\tau}/(1+ e^{\Lambda^k_\tau} ),$  for each $k\in \until{n}$

\smallskip 

        \nl set
        $q_k= \frac{\sqrt{\pi^1_k /\dist_k}}{\sum_{j=1}^n \sqrt{\pi^1_j/\dist_j}}$,
        for each $k\in \until{n}$\;
        
        \smallskip

        \nl  sample a region from probability distribution $(q_1,\ldots,q_n)$\;
        
        \smallskip

%        \nl at time $\tau\in\natural$, select a random region $k\in\mathcal{R}$\\
%        \hfill according to the probability distribution $\q^*$\;

        \nl collect sample $y_\tau$ from region $k$\;
        
        \smallskip 

        \nl update the CUSUM statistic at each region 

        \[
        \Lambda_{\tau}^{j}=
        \begin{cases}
          \Big ( \Lambda_{\tau-1}^{k} +\log\frac{f^1_{k}(y_{\tau})}{f^0_{k}(y_{\tau})}\Big)^+, & \text{if } j=k;\\
          \Lambda_{\tau-1}^{j}, & \text{if }
          j\in\until{n}\setminus\{k\};
        \end{cases}
        \]
        
 %      \nl update $\pi_{k}^1 = e^{\Lambda_{\tau}^{k}}/(1+e^{\Lambda_{\tau}^{k}})$\;
%        \emph{\% detect change if the threshold is crossed}

\smallskip
 
        \If{$\Lambda_{\tau}^{k} >\eta$}{
 \smallskip
 
          \nl anomaly detected at region
          $\mc R_k$\; 
\smallskip 

          \nl set $\Lambda_{\tau}^k =0$\;

        }
        
   		\nl set $\tau=\tau+1$ \;

\smallskip
   
        % \nl \quad \textbf{if} $\Lambda_{\tau}^{k} >\eta$, {\bf  then}  anomaly detected
        % at region $\mc R_k$\; and set $\Lambda_{\tau}^k =0$\\

}

 %\nl continue to step \algostep{4}

%  \nocaptionofalgo
    \caption{\textit{Single Vehicle Adaptive Ensemble CUSUM}}
  \label{algo:receding-horizon-vehicle-routing}}
\end{algorithm} 
\DecMargin{.3em}

% 		\nl partition regions in $m$ subsets with cardinality at most $\lceil n/m \rceil$\; 		
% 		
%\smallskip
%
%		\nl allocate one robot to each region \;
%		
%\smallskip		
%
%\nl \For {\textup{each partition} } {
Our adaptive ensemble CUSUM algorithm is formally presented in
Algorithm~\ref{algo:receding-horizon-vehicle-routing} for the single
vehicle case. For the case of multiple vehicles we resort to the
partitioning Algorithm~\ref{algo:partitioning} that implements the
single vehicle adaptive ensemble CUSUM
Algorithm~\ref{algo:receding-horizon-vehicle-routing} in each
partition. Let us denote the adaptive routing policy for a single
vehicle by ${\a}$ and the policy obtained from the partitioning
algorithm that implements single vehicle adaptive routing policy in each
partition by $\subscr{\a}{part}$.  We now analyze the performance of
the adaptive ensemble CUSUM algorithm. 
Since, the probability to visit any region varies with time in the 
adaptive ensemble CUSUM algorithm, we need to determine the
number of iterations between two consecutive visit to a region, i.e.,
the number of iterations for the recurrence of the region.   
We first derive a bound on the
expected number of samples to be drawn from a time-varying probability vector
for the recurrence of a particular state.

\begin{lemma}[\bit{Mean observations for region
    recurrence}]\label{lem:recurrence} Consider a sequence
  $\seqdef{x_\tau}{\tau\in \naturals}$, where
  $x_\tau$ is sampled from a probability vector $\boldsymbol{p}^{\tau}
  \in \Delta_{n-1}$.  If the $k$th entry of $\p^\tau$ satisfy
  $p^\tau_k \in (\alpha_k,\beta_k)$, for each $\tau\in \naturals$
  and some $\alpha_k,\beta_k \in (0,1)$, then the number of
  iterations $I_k$ for the recurrence of state $k$ satisfy $\expt[I_k]
  \le \beta_k/ \alpha_k^2$.
\end{lemma}
\begin{proof}
  The terms of the sequence $\seqdef{x_\tau}{\tau\in \naturals}$ are
  statistically independent. Further, the probability mass function
  $p^\tau$ is arbitrary. Therefore, the bound on the expected
  iterations for the first occurrence of state $k$ is also a bound on
  the subsequent recurrence of state $k$.  The expected number of
  iterations for first occurrence of region $k$ are
\[
\expt[I_k]= \sum_{i\in \naturals} i p^i_{k} \prod_{j=1}^{i-1}
(1-p_k^j)\le \beta_k \sum_{i\in \naturals} i (1-\alpha_k)^{i-1}
=\beta_k/\alpha_k^2.
\]
This establishes the statement.
\end{proof}

We utilize this upper bound on the expected number of iterations for recurrence
of a region to derive performance metrics for the adaptive ensemble CUSUM algorithm.
We now derive an upper bound on the expected detection delay at each region for 
adaptive ensemble CUSUM algorithm. We derive these bounds for the expected evolution
of the CUSUM statistic at each region.

\begin{theorem}[\bit{Adaptive ensemble CUSUM
    algorithm}] \label{thm:adaptive-ensemble-CUSUM} Consider the
  expected evolution of the CUSUM statistic at each region.  For the
  partitioning algorithm that implements single vehicle adaptive
  ensemble CUSUM algorithm
  (Algorithm~\ref{algo:receding-horizon-vehicle-routing}) in each
  subset of the partition, the following statement holds:
\begin{multline*}
\expt[\delta_k (\subscr{\a}{part})] \le  \Big(\frac{\etab}{\dist_k} + \frac{2(\lceil n/m \rceil -1)e^{\eta/2 }\sqrt{\dist_k}(1- e^{- \etab/2})}{\sqrt{\dist_{\min}}(1- e^{-\dist_k /2})} \\
  + \frac{(\lceil n/m \rceil-1)^2e^{\eta }{\dist_k}(1- e^{- \etab})}{{\dist_{\min}}(1- e^{-\dist_k })}\Big) (\Tbar_{\max}+d_{\max}) .
\end{multline*}
\end{theorem}
\begin{proof}
  We start by deriving expression for a single vehicle.  Let the number
  of iterations between the $(j\!-\!1)$th and $j$th visit to region
  $\mc R_k$ be $I^k_j$.

  Let the observation during the $j$th visit to region $\mc R_k$ be
  $y_j$ and the CUSUM statistic at region $\mc R_k$ after the visit be
  $C_j^k$.  It follows that the probability to visit region $\mc R_k$
  between $(j\!-\!1)$th and $j$th visit is greater than
\[
p^{j-1}_k = \frac{e^{C^k_{j-1}/2}/\sqrt{\dist_k}}{
  e^{C^k_{j-1}/2}/\sqrt{\dist_k}+(n-1)e^{\eta/2}/\sqrt{\dist_{\min}}}.
\]
Therefore, it follows from Lemma~\ref{lem:recurrence} that
\[
\expt[I^k_{j}] \le (1+(n-1) e^{(\eta-C^k_{j-1})/2} \sqrt{\dist_k/\dist_{\min}})^2.
\] 
Note that $C^k_j =\max\{0, C^k_{j-1} + \log
(f_k^1(y_j)/f_k^0(y_j))\}$. Since, maximum of two convex function is a
convex function, it follows from Jensen inequality~\cite{SIR:99}, that
\[
\expt[C^k_{j}] \ge \max\{0, \expt[C^k_{j-1}] +\dist_k\} \ge \expt[C^k_{j-1}] +\dist_k.
\] 
Therefore, $\expt[C^k_j] \ge j \dist_k$ and for expected evolution of the
CUSUM statistics 
\[
\expt[I^k_{j}] \le (1+(n-1) e^{(\eta- (j-1)\dist_k)/2} \sqrt{\dist_k/\dist_{\min}})^2.
\]  

Therefore, the total number of iterations $N_k$
required to collect $\supscr{N}{obs}_k$ observations at region $\mc
R_k$ satisfy
\begin{align*}
  \expt[N_k(\a) | \supscr{N}{obs}_k] &= \sum_{j=1}^{\supscr{N}{obs}_k}  (1+(n-1) e^{(\eta- (j-1)\dist_k)/2} \sqrt{\dist_k/\dist_{\min}})^2 \\
  &=\supscr{N}{obs}_k + \frac{2(n-1)e^{\eta/2} \sqrt{\dist_k}(1- e^{-\dist_k \supscr{N}{obs}_k/2})}{\sqrt{\dist_{\min}}(1- e^{-\dist_k /2})} \\
  &\qquad \qquad \quad +\frac{(n-1)^2e^{\eta} {\dist_k}(1- e^{-\dist_k \supscr{N}{obs}_k})}{{\dist_{\min}}(1- e^{-\dist_k})}.
\end{align*}
Note that the number of observations $\supscr{N}{obs}_k$ required at
region $\mc R_k$ satisfy $\expt[\supscr{N}{obs}_k]=\etab/\dist_k$.  It
follows from Jensen's inequality that
\begin{multline*}
\expt[N_k(\a)] \le \frac{\etab}{\dist_k} + \frac{2(n-1)e^{\eta/2 }\sqrt{\dist_k}(1- e^{- \etab/2})}{\sqrt{\dist_{\min}}(1- e^{-\dist_k /2})} \\
  + \frac{(n-1)^2e^{\eta }{\dist_k}(1- e^{- \etab})}{{\dist_{\min}}(1- e^{-\dist_k })}.
\end{multline*}
Since the expected time required to collect each evidence is smaller
$\Tbar_{\max}+d_{\max}$, it follows that
\[
\expt[\delta_k(\a)] \le (\Tbar_{\max}+d_{\max}) \expt[N_k(\a)].
\]
The expression for the partitioning policy that implement single vehicle
adaptive routing policy in each partition follow by substituting
$\lceil n/m \rceil$ in the above expressions.  This completes the
proof of the theorem.
\end{proof}
\begin{remark}[\bit{Performance bound}]
  The bound derived in Theorem~\ref{thm:adaptive-ensemble-CUSUM} is
  very conservative. Indeed, it assumes the CUSUM statistic at each
  region to be fixed at its maximum value $\eta$, except for the
  region in consideration. This is practically never the case. In
  fact, if at some iteration the CUSUM statistic is close to $\eta$,
  then it is highly likely that the vehicle visits that region at the
  next iteration, so that the updated statistic crosses the threshold
  $\eta$ and resets to zero.\oprocend
\end{remark}

\section{Numerical Results}\label{sec:numerical}
We now elucidate on the concepts developed in this paper through some
numerical examples.  We first validate the expressions for expected
detection delay obtained in
Section~\ref{sec:randomized-ensemble-cusum}.
\begin{example}[\bit{Expected detection delay}]\label{ex:single-delay}
  Consider a set of $4$ regions surveyed by a single vehicle. Let the location
  of the regions be $(10, 0), (5, 0), (0, 5)$, and $(0, 10)$,
  respectively. The vector of processing times at each region is
  $(1,2,3,4)$. Under the nominal conditions, the observations at each
  region are sampled from normal distributions $\mc N (0,1), \mc N
  (0,1.33), \mc N (0,1.67)$ and $\mc N(0,2)$, respectively, while
  under anomalous conditions, the observations are sampled from normal
  distributions with unit mean and same variance as in nominal case.
  Let the prior probability of anomaly at each region be $0.5$.  An
  anomaly appears at each region at time $50, 200, 350,$ and $500$,
  respectively. Assuming that the vehicle is holonomic and moves at unit
  speed, the expected detection delay at region $\mc R_1$ and the
  average detection delay are shown in
  Fig.~\ref{fig:detection_delay_single}. It can be seen that the
  theoretical expressions provide a lower bound to the expected
  detection delay obtained through Monte-Carlo simulations. This
  phenomenon is attributed to the Wald's approximation. \oprocend
\end{example}

\begin{figure}
    \centering \scriptsize
    \includegraphics[width=1\columnwidth]{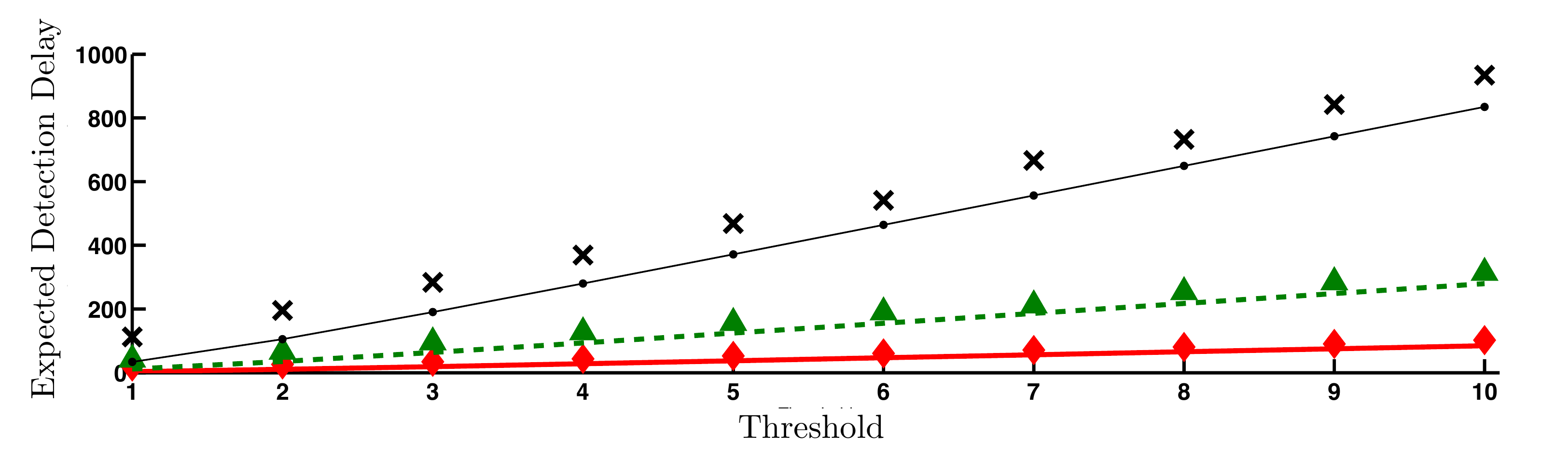}\\
    (a) Expected detection delay at region $\mc R_1$
     \includegraphics[width=1\columnwidth]{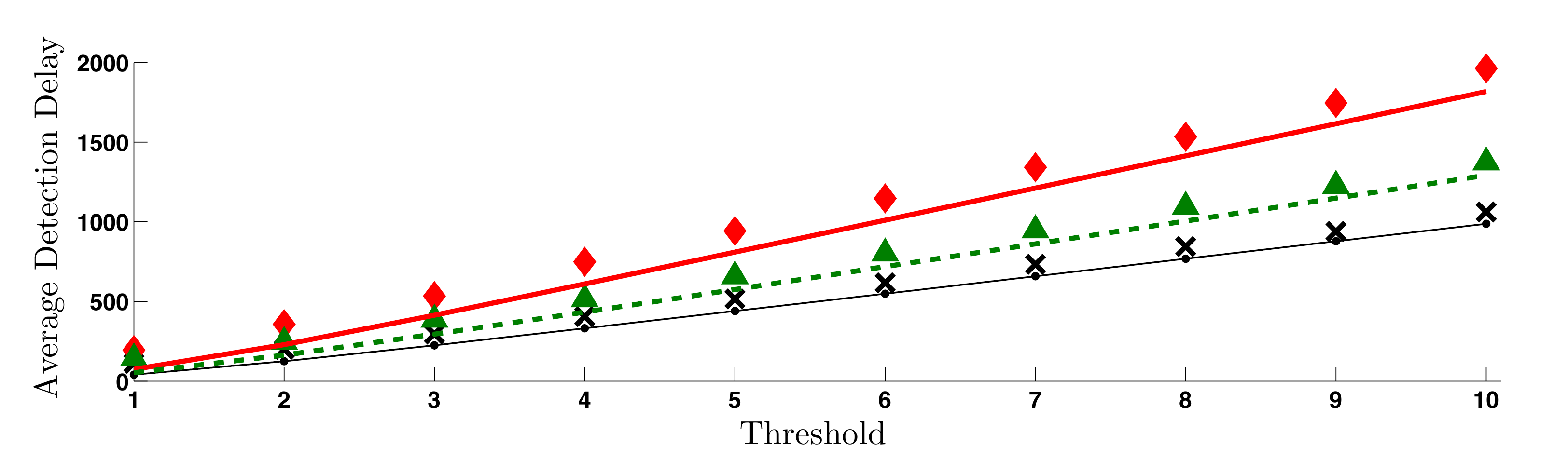}\\
     (b) Average detection delay
    \caption{Expected and average detection delay. Solid black line with dots and black $\times$, respectively, represent the theoretical expression and 
    the value obtained by Monte-Carlo simulations under stationary policy $\q=[\;0.2 \; 0.25\; 0.25 \; 0.3\;]$. Dashed green line and green triangles,
    respectively, represent the theoretical expression and 
    the value obtained by Monte-Carlo simulations under stationary policy $\q=[\;0.5\; 0.2\; 0.2\; 0.1\;]$. Solid red line and red diamonds, 
    respectively, represent the theoretical expression and 
    the value obtained by Monte-Carlo simulations under stationary policy $\q=[\;0.85\; 0.05\; 0.05\; 0.05\;]$.  }
    \label{fig:detection_delay_single}
\end{figure}

We remarked in Section~\ref{sec:setup} that if each region cannot be
reached from another region in a single hop, then a fastest mixing
Markov chain (FMMC) with the desired stationary distribution can be
constructed.  Consider a set of regions modeled by the graph $\mc
G=(V, \mc E)$, where $V$ is the set of nodes (each node corresponds to
a region) and $\mc E$ is the set of edges representing the
connectivity of the regions.  The transition matrix of the FMMC $P\in
\real^{n\times n}$ with a desired stationary distribution $\q \in
\Delta_{n-1}$ can be determined by solving the following convex
minimization problem~\cite{SB-PD-LX:04}:
\begin{equation*}
\begin{split}
\minimize & \quad \| Q^{1/2} P Q^{1/2} -\subscr{\q}{root} \subscr{\q^T}{root} \|_2\\
\subject & \quad P \boldsymbol{1} =\boldsymbol{1}\\
& \quad QP=P^TQ \\
&\quad P_{ij}\ge 0, \text{ for each } (i,j) \in \mc E \\
&\quad P_{ij}=0, \text{ for each } (i,j)\notin \mc E,
\end{split}
\end{equation*}
where $Q$ is a diagonal matrix with diagonal $\q$,
$\subscr{\q}{root}=(\sqrt{q_1},\ldots, \sqrt{q_n})$, and $
\boldsymbol{1}$ is the vector of all ones. We now demonstrate the
effectiveness of FMMC in our setup.
\begin{example}[\bit{Effectiveness of FMMC}]\label{ex:metro}
  Consider the same set of data as in
  Example~\ref{ex:single-delay}. We study the expected and average
  detection delay for randomized ensemble CUSUM algorithm when the
  regions to visit are sampled from the FMMC.  The expected and
  average detection delay for all-to-all connection topology, line
  connection topology and ring connection topology are shown in
  Fig.~\ref{fig:detection_delay_metro}.  It can be seen that the
  performance under all three topologies is remarkably close to each
  other. \oprocend
\end{example}
\begin{figure}\scriptsize
    \centering
        \includegraphics[width=1\columnwidth]{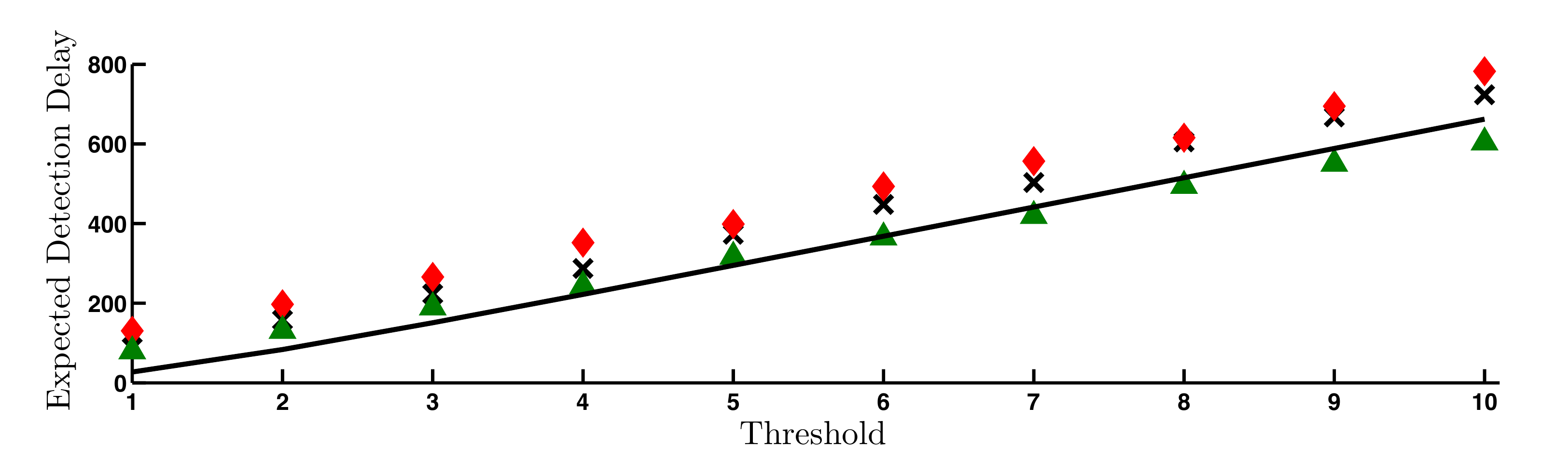}\\
    (a) Expected detection delay at region $\mc R_1$
      \includegraphics[width=1\columnwidth]{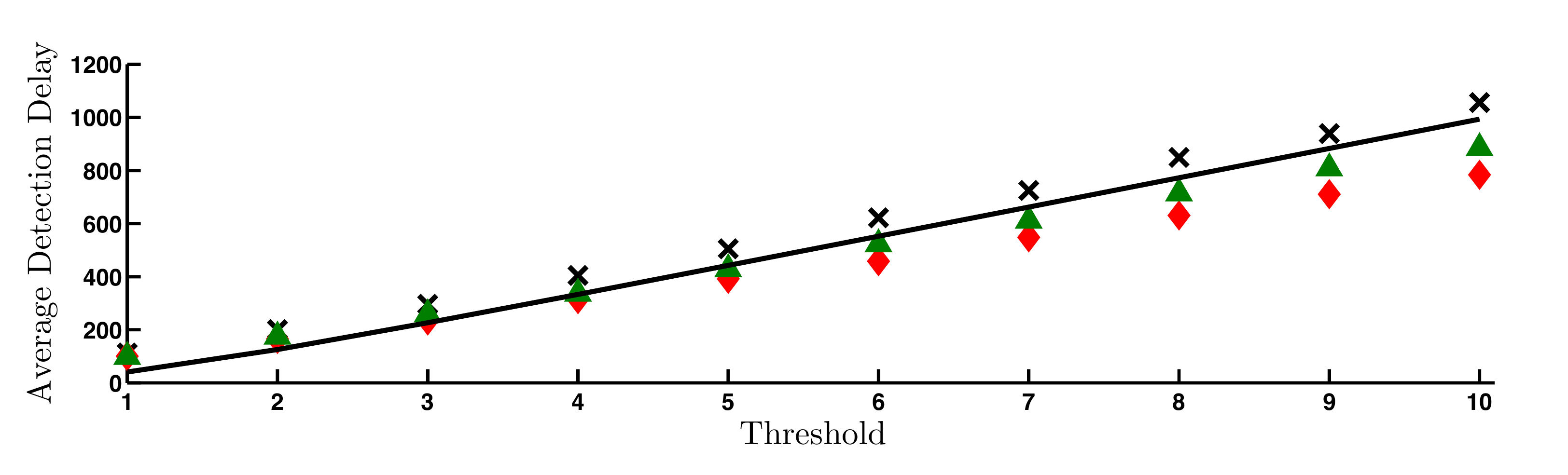}\\
    (b) Average detection delay
    \caption{Expected and average detection delay for uniform
      stationary policy. The solid black line represents the
      theoretical expression.  The black $\times$, red diamonds, and
      green triangles, respectively, represent the values obtained by
      Monte-Carlo simulations for all-to-all, line, and ring
      connection topology. For the line and ring topologies, the region
      to visit at each iteration is sampled from the fastest mixing
      Markov chain with the desired stationary distribution.}
    \label{fig:detection_delay_metro}
\end{figure}

We now study the performance of the (numerically computed) optimal and
our efficient stationary policies for the single vehicle randomized
ensemble CUSUM algorithm.
\begin{example}[\bit{Single vehicle optimal stationary policy}]
  For the same set of data as in Example~\ref{ex:single-delay}, we now
  study the performance of the uniform, the (numerically computed)
  optimal and our efficient stationary routing policies. A comparison
  is shown in Fig.~\ref{fig:single-vehicle-comparison}. Notice that
  the performance of the optimal and efficient stationary policy is
  extremely close to each other. \oprocend
\end{example}

\begin{figure}
 \includegraphics[width=1\columnwidth]{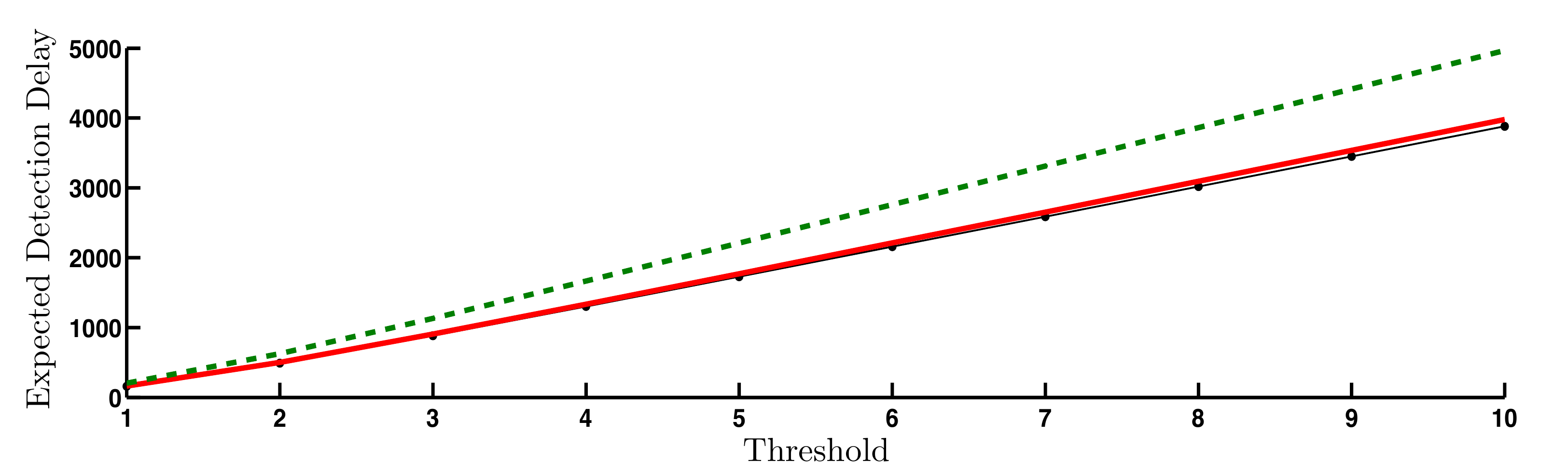}
 \caption{Average detection delay for a single vehicle. The solid red
   line, the dashed green line, and the solid black line with dots
   represent efficient, uniform, and optimal stationary policies,
   respectively.}
\label{fig:single-vehicle-comparison}
\end{figure}

We now study the performance of the optimal, partitioning and uniform
stationary policies for randomized ensemble CUSUM algorithm with
multiple vehicles.
\begin{example}[\bit{Multiple-vehicle optimal stationary policy}]
  Consider a set of $6$ regions surveyed by $3$ vehicles.  Let the
  regions be located at $(10 ,0), (5, 0), (0, 5), (0, 10), (0, 0)$ and
  $(5, 5)$.  Let the processing time at each region be unitary. Under
  nominal conditions, the observations at each region are sampled from
  normal distributions $\mc N (0,1)$, $\mc N (0,1.4)$, $\mc N
  (0,1.8)$, $\mc N (0,2.2)$, $\mc N (0,2.6)$ and $\mc N(0,3)$,
  respectively. Under anomalous conditions, the observations are
  sampled from normal distributions with unit mean and same variance
  as in the nominal case.  Let the prior probability of anomaly at
  each region be $0.5$.  An anomaly appears at each region at time
  $25, 35, 45, 55 , 65 $ and $75$, respectively.  Assuming that the
  vehicles are holonomic and moves at unitary speed, the average
  detection delay for the uniform stationary policy for each vehicle,
  the partitioning policy in which each vehicle implements single
  vehicle efficient stationary policy in each subset of the partition,
  and the partitioning policy in which each vehicle implements single
  vehicle optimal stationary policy in each subset of the partition is
  shown in Fig.~\ref{fig:multiple-vehicle-comparison}.  \oprocend
\end{example}

\begin{figure}
    \centering
    \includegraphics[width=1\columnwidth]{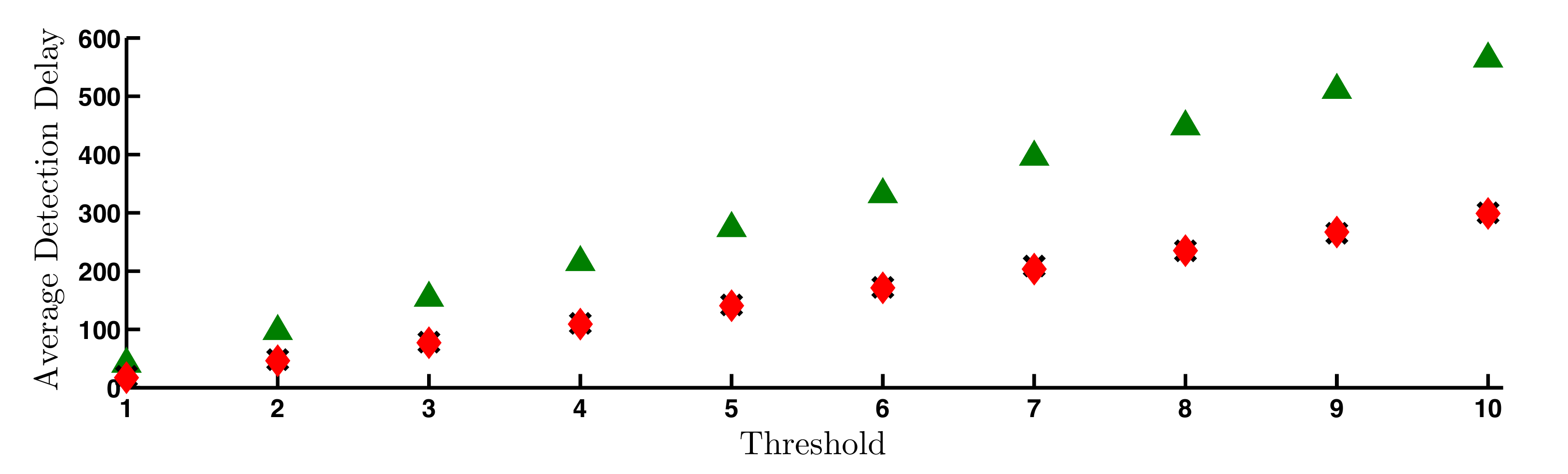}
    \caption{Average detection delay for $3$ vehicles surveying $6$
      regions. The green triangles represent the policy in which each
      vehicle surveys each region uniformly. The red diamonds and black
      $\times$ represent the partitioning policy in which each vehicle
      implements the single vehicle efficient stationary policy and the
      single vehicle optimal stationary policy, respectively.}
    \label{fig:multiple-vehicle-comparison}
\end{figure}

We now study the performance of the adaptive ensemble CUSUM algorithm,
and we numerically show that it improves the performance of our
stationary policy.
\begin{example}[\bit{Adaptive ensemble CUSUM
    algorithm}] \label{ex-adaptive-ensemble} Consider the same set of
  regions as in Example~\ref{ex:single-delay}. Let the processing time
  at each region be unitary. The observations at each region are
  sampled from normal distributions $\mc N(0,\sigma^2)$ and $\mc
  N(1,\sigma^2)$, in nominal and anomalous conditions, respectively.
  Under the nominal conditions at each region and $\sigma^2=1$, a
  sample evolution of the adaptive ensemble CUSUM algorithm is shown
  in Fig.~\ref{fig:adaptive-sample}(a).  The anomaly appears at
  regions $\mc R_2$, $\mc R_3$, and $\mc R_4$ at time $100$, $300$,
  and $500$, respectively. Under these anomalous conditions and
  $\sigma^2=1$, a sample evolution of the adaptive ensemble CUSUM
  algorithm is shown in Fig.~\ref{fig:adaptive-sample}(b).  It can be
  seen that the adaptive ensemble algorithm samples a region with high
  likelihood of anomaly with high probability, and, hence, it improves
  upon the performance of the stationary policy.

  We now study the expected detection delay under adaptive ensemble
  CUSUM algorithm and compare it with the efficient stationary policy.
  The anomaly at each region appears at time $50, 200, 350$ and $500$,
  respectively.  The expected detection delay obtained by Monte-Carlo
  simulations for $\sigma^2=1$ and different thresholds is shown in
  Fig.~\ref{fig:adaptive-performance}(a). It can be seen that the
  adaptive policy improves the detection delay significantly over the
  efficient stationary policy for large thresholds. It should be noted
  that the detection delay minimization is most needed at large
  thresholds because the detection delay is already low at small
  thresholds. Furthermore, frequent false alarms are encountered at
  low thresholds and hence, low thresholds are not typically chosen.
  The expected detection delay obtained by Monte-Carlo simulations for
  different value of $\sigma^2$ and threshold $\eta=5$ is shown in
  Fig.~\ref{fig:adaptive-performance}(b). Note that for a given value
  of $\sigma^2$, the Kullback-Leibler divergence between $\mc
  N(1,\sigma^2)$ and $\mc N(0,\sigma^2)$ is $1/2\sigma^2$.  It can be
  seen that the adaptive policy improves the performance of the
  stationary policy for each value of noise.  \oprocend
\end{example}

\begin{figure}\scriptsize
\centering
 \includegraphics[width=1\columnwidth]{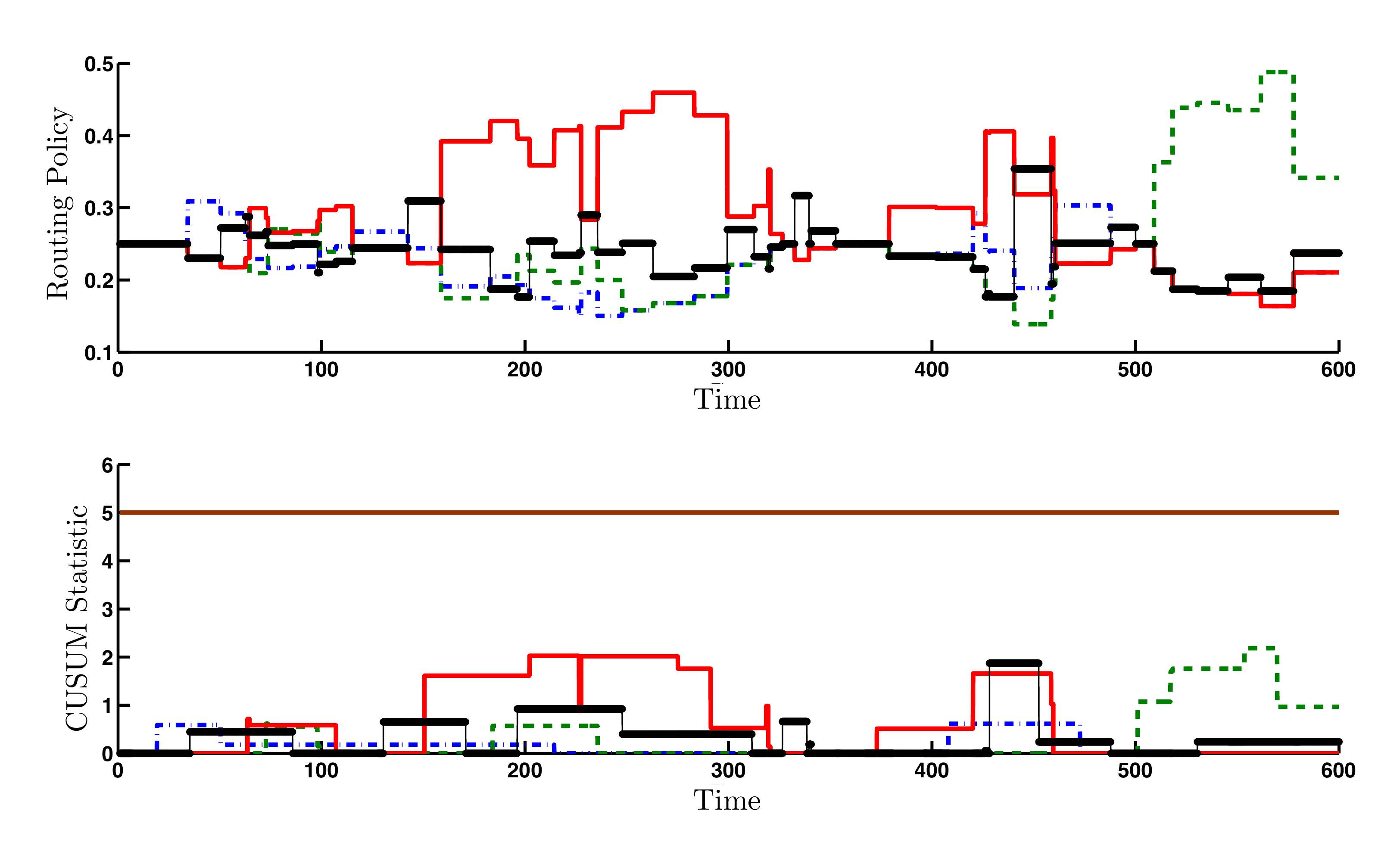}\\
(a) CUSUM statistic and vehicle routing probabilities under nominal conditions\\
 \includegraphics[width=1\columnwidth]{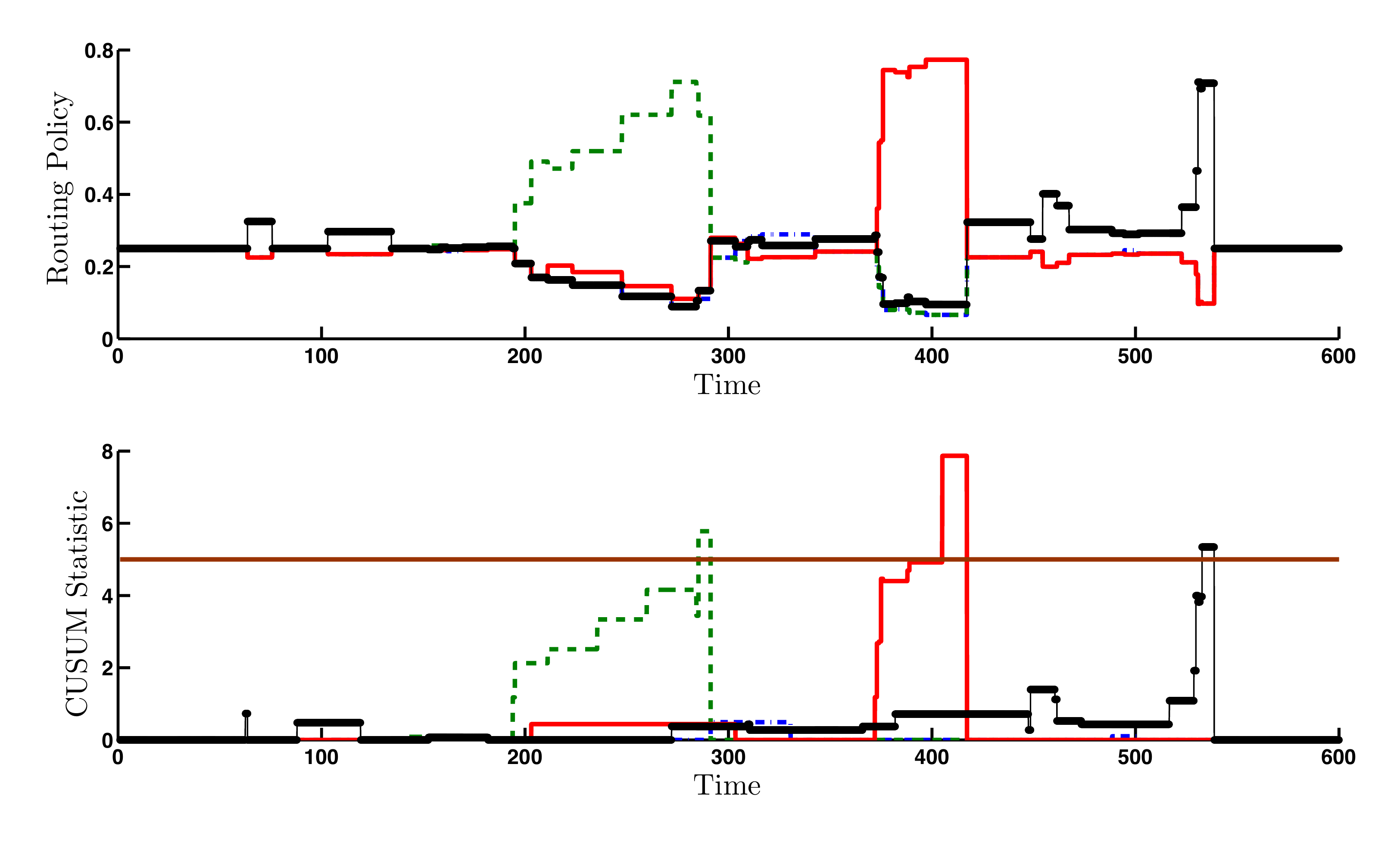}\\
(b) CUSUM statistic and vehicle routing probabilities under anomalous conditions\\
\caption{Sample evolution of the adaptive ensemble CUSUM algorithm.
The dashed-dotted blue line, dashed green line, solid red line and  solid black line with dots represent data from regions $\mc R_1, \mc R_2, \mc R_3$ and $\mc R_4$, respectively. 
The solid brown horizontal line represents the threshold. The vehicle
routing probability is a function of the likelihood of anomaly at each
region. As the likelihood of an anomaly being present at a region
increases, also the probability to survey that region increases. 
Anomalies appear at region $\mc R_2$, $\mc R_3$ and $\mc R_4$ at times $100$, $300$ and $500$, respectively. Once an anomaly is detected, it is removed and the statistic is reset to zero.
\label{fig:adaptive-sample}}
\end{figure}

\begin{figure}\scriptsize
\centering
 \includegraphics[width=1\columnwidth]{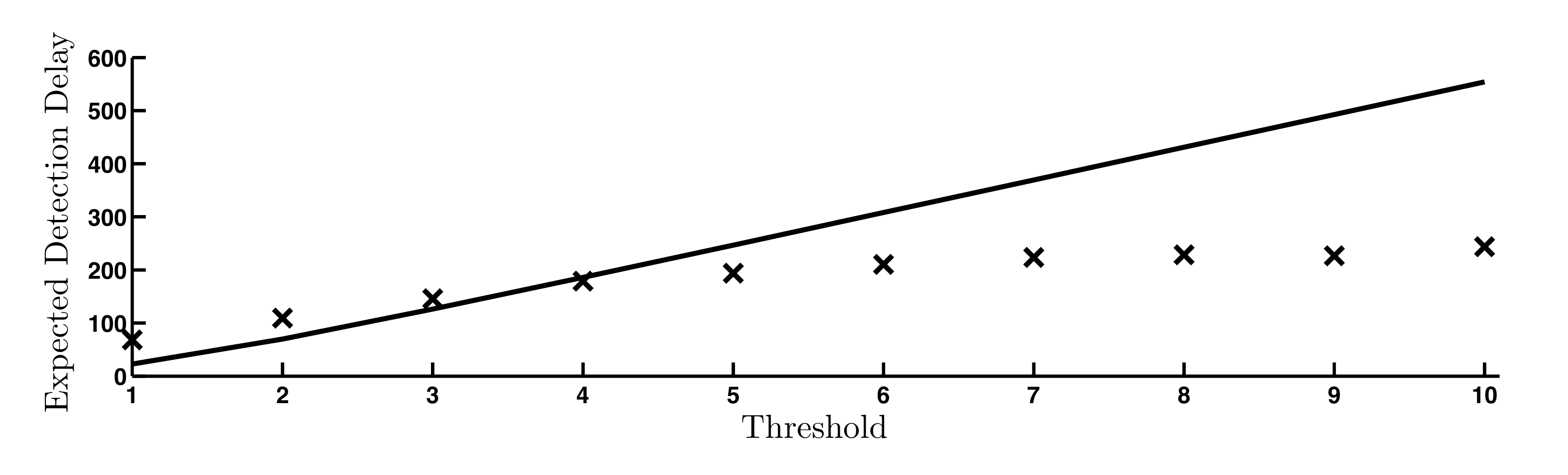}\\
(a) Expected detection delay as a function of threshold\\
 \includegraphics[width=1\columnwidth]{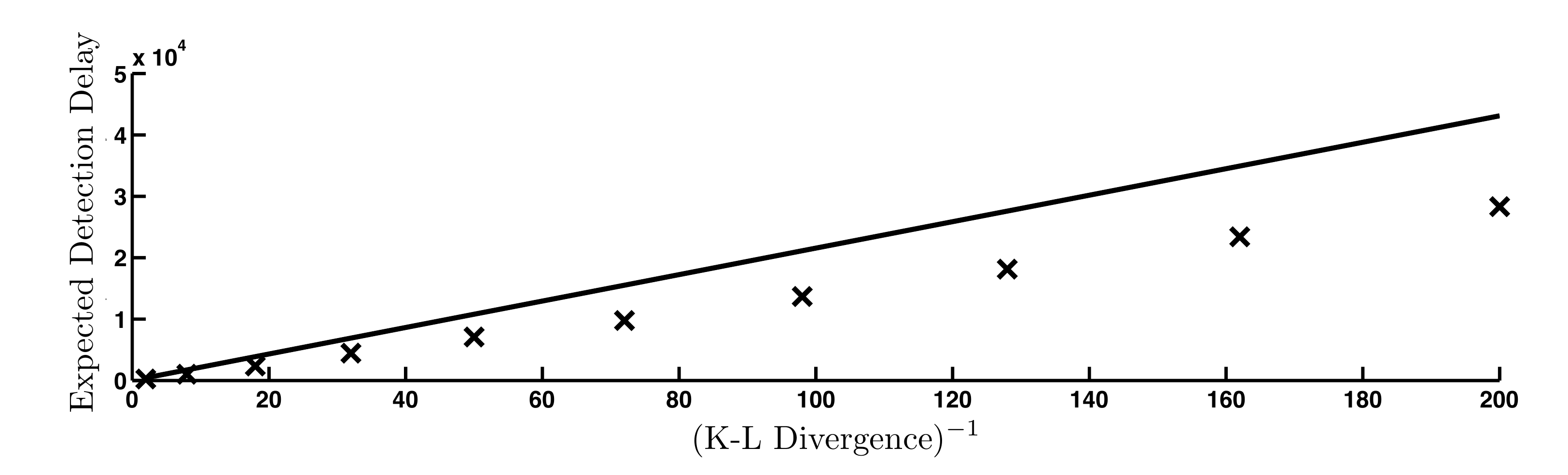}\\
(b) Expected detection delay as a function of KL divergence
\caption{Performance of the adaptive ensemble CUSUM algorithm. The solid black line represents
the theoretical expected detection delay for the efficient stationary policy and the black $\times$ represent
the expected detection delay for the adaptive ensemble CUSUM algorithm. \label{fig:adaptive-performance}}
\end{figure}

We now apply the adaptive ensemble CUSUM algorithm to a more general
scenario where the anomalous distribution is not completely known. As
remarked in Section~\ref{sec:setup}, in this case, the CUSUM algorithm
should be replaced with the GLR algorithm. Given the nominal
probability density function $f^0_k$ and the anomalous probability
density function $f^1_k(\cdot| \theta)$ parameterized by $\theta \in
\Theta \subseteq \real^\ell$, for some $\ell \in \naturals$, the GLR
algorithm~\cite{MB-IVN:93}, works identically to the CUSUM algorithm,
except that the CUSUM statistic is replaced by the statistic
\[
\Lambda^k_\tau = \max_{t\in\until{\tau}} \sup_{\theta\in \Theta}
\sum_{i=t}^\tau \log \frac{f^1_k(y_i|\theta)}{f^0_k(y_i)}.
\]
\begin{example}[\bit{Generalized Likelihood Ratio}]
  For the same set of data as in Example~\ref{ex-adaptive-ensemble},
  assume that there are three types of potential anomalies at each
  region. Since any combination of these anomalies can occur
  simultaneously, there are $7$ potential distributions under
  anomalous conditions. We characterize these distributions as
  different hypothesis and assume that the observations under each
  hypothesis $h\in \until{8}$ are sampled from a normal distribution
  with mean $\mu_h$ and covariances $\Sigma_h$. Let
\begin{align*}
\mu_1&= \left[\begin{smallmatrix} 0\\ 0\\0 \end{smallmatrix}\right], 
\mu_2= \left[\begin{smallmatrix} 1\\ 0\\0 \end{smallmatrix}\right], 
\mu_3= \left[\begin{smallmatrix} 0\\ 1\\0 \end{smallmatrix}\right], 
\mu_4= \left[\begin{smallmatrix} 0\\ 0\\1 \end{smallmatrix}\right], \\
\mu_5&= \left[\begin{smallmatrix} 1\\ 1\\0 \end{smallmatrix}\right], 
\mu_6= \left[\begin{smallmatrix} 0\\ 1\\1 \end{smallmatrix}\right], 
\mu_7= \left[\begin{smallmatrix} 1\\ 0\\1 \end{smallmatrix}\right], 
\mu_8= \left[\begin{smallmatrix} 1\\ 1\\1 \end{smallmatrix}\right], \text{ and}
\end{align*}
\begin{align*}
\Sigma_1&= \left[\begin{smallmatrix} 1 & 0 & 0\\ 0 & 1 & 0\\ 0 & 0 & 1 \end{smallmatrix}\right], 
\Sigma_2= \left[\begin{smallmatrix} 2 & 1 & 0\\ 1 & \frac{3}{2} & 0\\ 0 & 0 & 1 \end{smallmatrix}\right], 
\Sigma_3=\left[\begin{smallmatrix} 1 & 1 & 0\\ 1 & 2 & 1\\ 0 & 1 & \frac{3}{2} \end{smallmatrix}\right], 
\Sigma_4= \left[\begin{smallmatrix} \frac{3}{2} & 0 & 0\\ 0 & 1 & 1\\ 0 & 1 & 2 \end{smallmatrix}\right], \\
\Sigma_5&= \left[\begin{smallmatrix} 2 & 1 & 0\\ 1 & 2 & 1\\ 0 & 1 & 1 \end{smallmatrix}\right], 
\Sigma_6=\left[\begin{smallmatrix} 1 & 1 & 0\\ 1 & 2 & 1\\ 0 & 1 & 2 \end{smallmatrix}\right], 
\Sigma_7= \left[\begin{smallmatrix} 2 & 0 & 1\\ 0 & 1 & 1\\ 1 & 1 & 2 \end{smallmatrix}\right], 
\Sigma_8= \left[\begin{smallmatrix} 2 & 1 & 1\\ 1 & 2 & 1\\ 1 & 1 & 2 \end{smallmatrix}\right].
\end{align*}
We picked region $\mc R_1$ as non-anomalous, while hypothesis $4$, $6$, and
$8$ were true at regions $\mc R_2, \mc R_3$, and $\mc R_4$, respectively.
The Kullback-Leibler divergence at a region was chosen as the minimum of
all possible Kullback-Leibler divergences at that region.  A sample
evolution of the adaptive ensemble CUSUM algorithm with GLR statistic
replacing the CUSUM statistic is shown in Fig~\ref{fig:GLR-sample}(a). It
can be seen the performance is similar to the performance in
Example~\ref{ex-adaptive-ensemble}.  As an additional ramification of this
algorithm, we also get the likelihood of each hypothesis at each region. It
can be seen in Fig~\ref{fig:GLR-sample}(b) that the true hypothesis at each
region corresponds to the hypothesis with maximum likelihood. \oprocend
\end{example}

\begin{figure}\scriptsize
\centering
    \includegraphics[width=1\columnwidth]{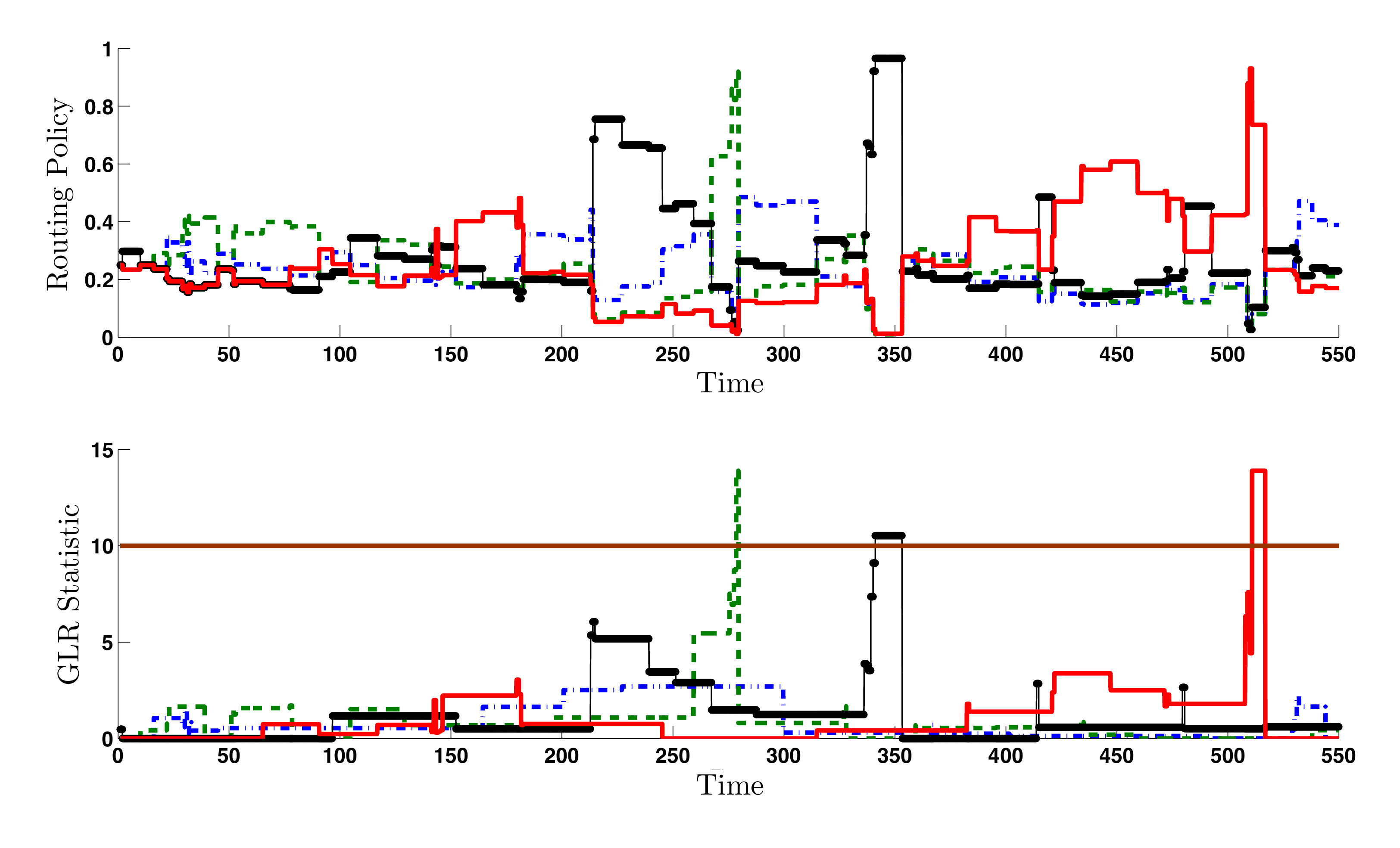}
(a) GLR statistic under anomalous conditions\\
 \includegraphics[width=1\columnwidth]{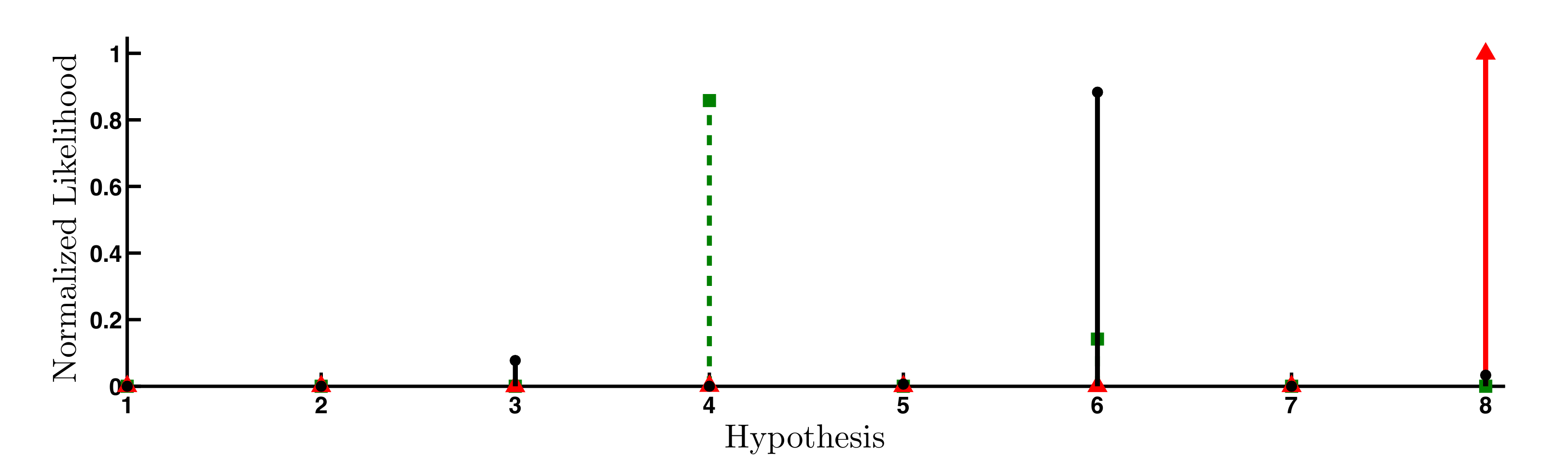}
(b) Normalized likelihood of each hypothesis
\caption{Sample evolution of the adaptive ensemble CUSUM algorithm with GLR
  statistic.  The dashed-dotted blue line, dashed green line, solid red
  line and solid black line with dots represent data from regions $\mc R_1,
  \mc R_2, \mc R_3$ and $\mc R_4$, respectively.  The solid brown
  horizontal line represents the threshold. The vehicle routing probability
  is a function of the likelihood of anomaly at each region. As the
  likelihood of an anomaly being present at a region increases, also the
  probability to survey that region increases. Anomalies appear at region
  $\mc R_2$, $\mc R_3$ and $\mc R_4$ at times $100$, $300$ and $500$,
  respectively. Once an anomaly is detected, it is removed and the
  statistic is reset to zero. The true hypothesis at each region
  corresponds to the hypothesis with maximum likelihood
\label{fig:GLR-sample}}
\end{figure}

\section{Experimental Results}\label{sec:experimental-results}
We first detail our implementation of the algorithms using the Player/Stage robot control software
package and the specifics of our robot hardware.  We then present the
results of the experiment.
% One of the strengths of Player/Stage is that
%the same code can run on either a real or simulated robot.

\subsubsection*{Robot hardware}
We use Erratic mobile robots from Videre Design shown in
Fig. \ref{fig:robot}.  The robot platform has a roughly square
footprint (40cm $\times$ 37cm), with two differential drive wheels and
a single rear caster. Each robot carries an on-board computer with a
1.8Ghz Core 2 Duo processor, 1 GB of memory, and 802.11g wireless
communication. For navigation and localization, each robot is equipped
with a Hokuyo URG- 04LX laser rangefinder. The rangefinder scans 683
points over 240\degree \;at 10Hz with a range of 5.6 meters.
% For simulations, the virtual robots are modeled off of our hardware.

\begin{figure}[t]
\centering
\includegraphics[width=.9\columnwidth]{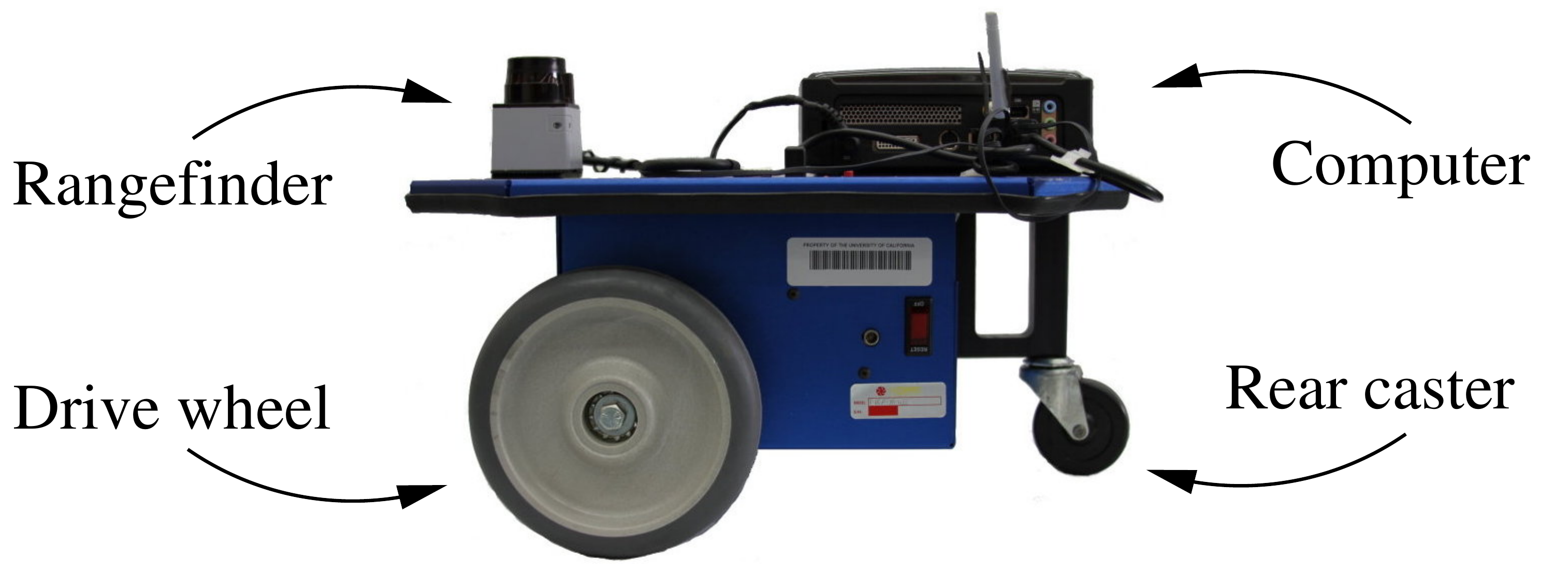}
\caption{Erratic mobile robot with URG-04LX laser rangefinder.}
\label{fig:robot}
\end{figure}

\subsubsection*{Localization}
We use the \textit{amcl} driver in Player which implements Adaptive
Monte-Carlo Localization~\cite{DF-WB-FD-ST:01}. The physical robots
are provided with a map of our lab with a 15cm resolution and told
their starting pose within the map (Fig. \ref{fig:exp_lab}). We set an
initial pose standard deviation of 0.9m in position and 12\degree \;
in orientation, and request updated localization based on 50 of the
sensors range measurements for each change of 2cm in robot position or
2\degree\; in orientation.  We use the most likely pose estimate by
amcl as the location of the robot.
% We let simulated robots access perfect localization information.

\subsubsection*{Navigation}
Each robot uses the \textit{snd} driver in Player for the Smooth
Nearness Diagram navigation~\cite{JWD-FB:08a}. For the hardware, we
set the robot radius parameter to 22cm, obstacle avoidance distance to
0.5m, and maximum speed to $0.2$m/s. 
% For our simulation, we set the maximum speed to $5$m/s.
We let a robot achieve its target when it is within 10cm of the
target.

\subsubsection*{Experiment setup}
\begin{figure}
    \centering
    \includegraphics[width=.7\columnwidth]{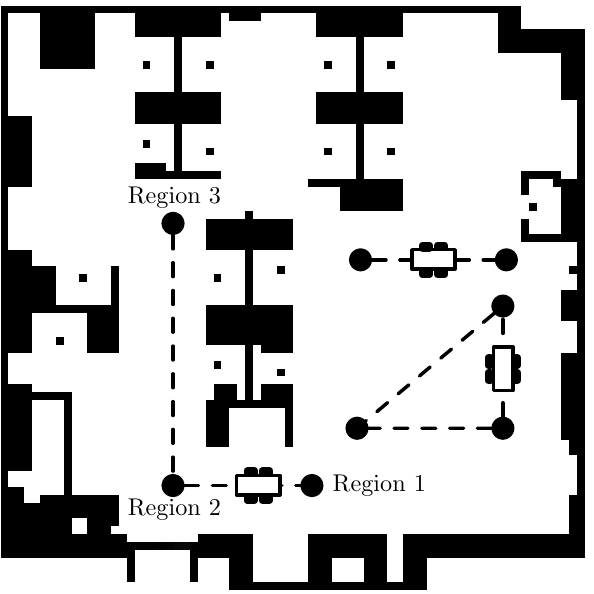}
    \caption{This figure shows a map of our lab together with our
      surveillance configuration. Three erratic robots survey the
      selected $8$ regions (black dots), which have been partitioned
      among the robots. Regions $1$, $2$, and $3$ are also considered
      in Fig. \ref{fig:experiment}, where we report the statistics of
      our detection algorithm.}
    \label{fig:exp_lab}
\end{figure}

For our experiment we employed our team of $3$ Erratic robots to
survey our laboratory. As in Fig. \ref{fig:exp_lab}, a set of $8$
important regions have been chosen and partitioned among the
robots. Each robot surveys its assigned regions. In
particular, each robot implements the single robot adaptive ensemble
CUSUM algorithm in its regions. Notice that Robot~$1$ cannot travel
from region~$1$ to region~$3$ in a single hop. Therefore, Robot~$1$
selects the regions according to a Markov chain with desired
stationary distribution. This Markov chain was constructed using the
Metropolis-Hastings algorithm. In particular, for a set of regions
modeled as a graph $\mc G=(V,\mc E)$, to achieve a desired stationary
routing policy $\q$, the Metropolis-Hastings algorithm~\cite{LW:04},
picks the transition matrix $P$ with entries:
\[
P_{ij}=\begin{cases}
  0, & \text{if } (i,j) \notin \mc E,\\
  \min \big\{\frac{1}{d_i}, \frac{q_j}{q_i d_j}\big\} & \text{if } (i,j) \in \mc E \text{ and } i\ne j,\\
  1-\sum_{k=1, k\ne i}^n P_{ik} & \text{if } (i,j) \in \mc E \text{
    and } i= j,
\end{cases}
\]
where $d_i$ is the number of regions that can be visited from region $\mc R_i$.
\begin{figure}
    \centering
    \includegraphics[width=.9\columnwidth]{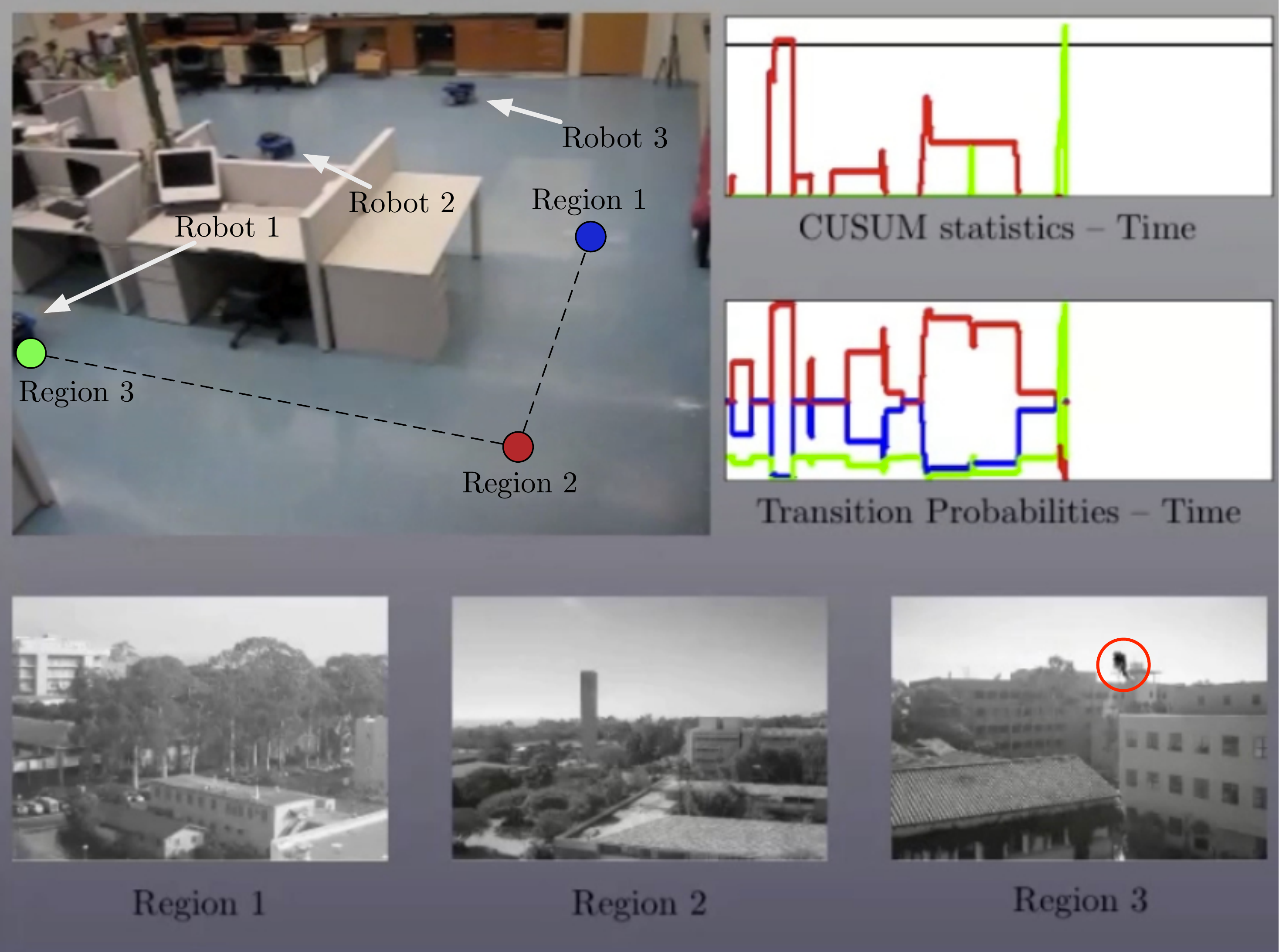}
    \caption{A snapshot of our surveillance experiment, where three
      robots survey six locations in our lab
      (Fig. \ref{fig:exp_lab}). In this figure we show the three
      regions assigned to the first robot. Each region correspond to a
      part of our campus, and observations are taken
      accordingly. Notice that Region 3 contains an anomaly (black
      smoke), and that the CUSUM statistics, which are updated upon
      collection of observations, reveal the anomaly (green peak). The
      transition probabilities are updated according to our adaptive
      ensemble CUSUM algorithm.}
    \label{fig:experiment}
\end{figure}

Observations (in the form of pictures) are collected by a robot each
time a region is visited. In order to have a more realistic
experiment, we map each location in our lab to a region in our
campus. Then, each time a robot visit a region in our lab, a picture
of a certain region in our campus is selected as observation (see
Fig. \ref{fig:experiment}). Pictures have been collected prior to the
experiment. 

\begin{figure}
    \centering
    \includegraphics[width=.9\columnwidth]{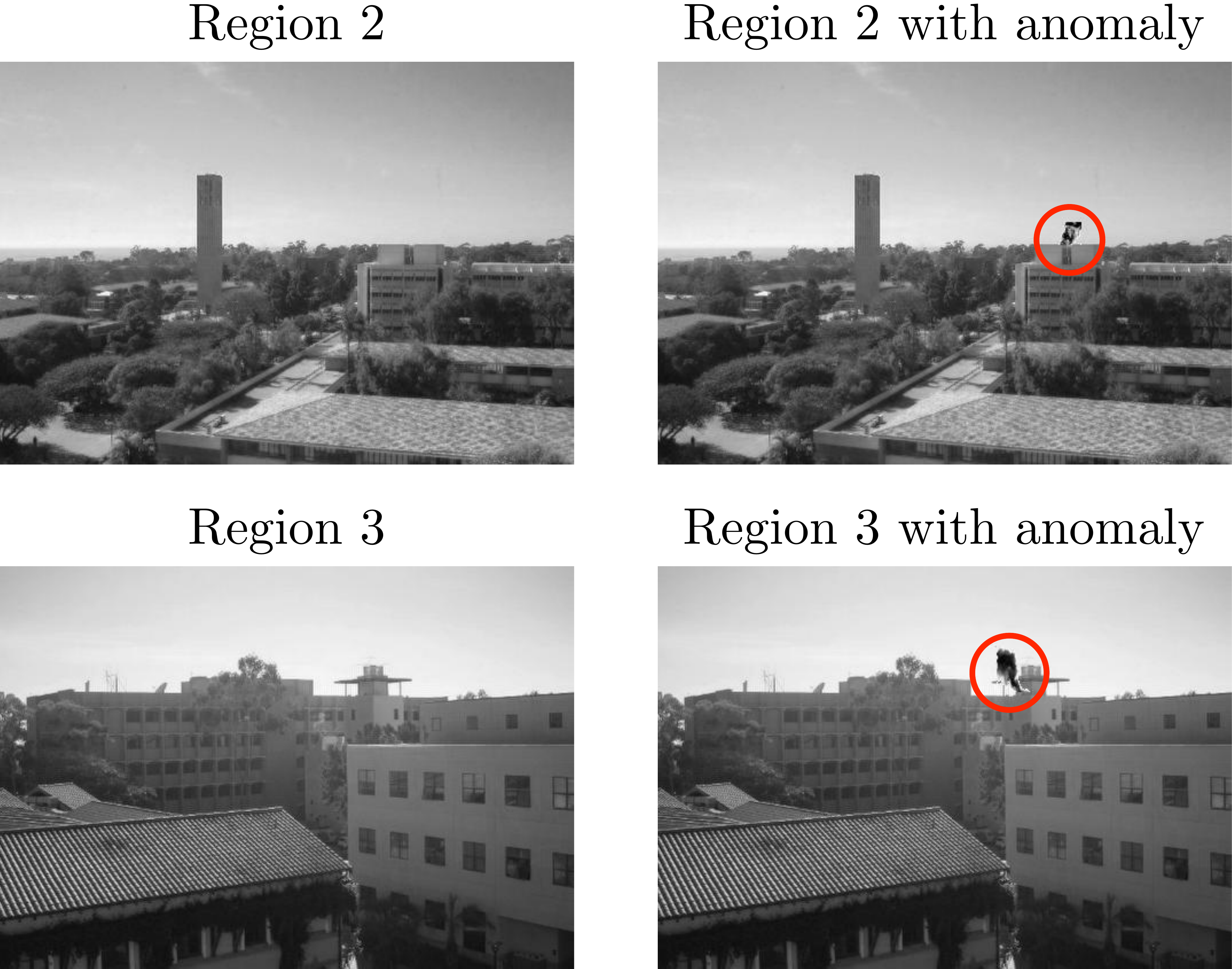}
    \caption{This figure shows sample pictures from Region 2 and Region
      3, both with and without the anomaly to be detected.}
    \label{fig:anomaly}
\end{figure}

Finally, in order to demonstrate the effectiveness of our anomaly
detection algorithm, some pictures from regions 2 and 3 have been
manually modified to contain an anomalous pattern; see
Fig. \ref{fig:anomaly}. Anomalous pictures are collected by Robot 1 at
some pre-specified time instants (the detection algorithm, however,
does not make use of this information).

\subsubsection*{Probability density function estimation}
In order to implement our adaptive ensemble CUSUM algorithm, the
probability density functions of the observations at the regions in presence and absence 
of  an anomaly need to be estimated. For this task, we first
collect sample images, and we register them in order to align their
coordinates~\cite{RJR-SA-OA-BR:05}. We then select a reference image,
and compute the difference between the sample pictures and the
reference image. Then, we obtain a coarse representation of each
difference image by dividing the image into blocks. For each
difference image, we create a vector containing the mean value of each
block, and we compute the mean and standard deviation of these
vectors. Finally, we fit a normal distribution to represent the
collected nominal data. In order to obtain a probability density
distribution of the images with anomalies, we manually modify the
nominal images, and we repeat the same procedure as in the nominal
case.

\subsubsection*{Experiment results}
The results of our experiment are illustrated in
Fig. \ref{fig:experiment}, Fig. \ref{fig:anomaly}, and in the
multimedia extension available at {\tt http://www.ijrr.org}. From the
CUSUM statistics we note that the anomalies in Region 2 and Region 3
are both detected: indeed both the red curve and the green curve pass
the decision threshold. We also note that few observations are
necessary to detect the anomaly. Since the robots successfully survey
the given environment despite sensor and modeling uncertainties due to
real hardware, we conclude that our modeling assumptions in Section
\ref{sec:setup} are not restrictive.

\section{Conclusions} \label{sec:conclusions} In this paper we studied
a spatial quickest detection problem in which multiple vehicles
surveil a set of regions to detect anomalies in minimum time. We
developed a novel ensemble CUSUM algorithm to detect an anomaly in any
of the regions. A stochastic vehicle routing policy was adopted in
which the vehicle samples the next region to visit from a probability
vector. In particular, we studied (i) stationary policy: the
probability vector is a constant function of time; and (ii) adaptive
policy: the probability vector is adapted with time based on the
collected observations.  We designed an efficient stationary policy
that depends on the travel time of the vehicles, the processing time
required to collect information at each region, and the anomaly
detection difficulty at each region. In adaptive policy, we modified
the efficient stationary policy at each iteration to ensure that the
regions with high likelihood of anomaly are visited with high
probability, and thus, improved upon the performance of the stationary
policy.  We also mentioned the methods that extend the ideas in this
paper immediately to the scenario in which the distributions of the
observations in presence and absence of anomaly are not completely
known, but belong to some parametrized family, or to the scenario in
which the observations collected from each region are not independent
(e.g., in the case of dynamic anomalies).

There are several possible extensions of the ideas considered here. First,
in the case of dependent observations at each region, the current method
assumes known distributions in presence and absence of anomalies. An
interesting direction is to design quickest detection strategies that are
robust to the uncertainties in these distributions. Second, the anomalies
considered in this paper are always contained in the same region. It would
be of interest to consider anomalies that can move from one region to
another.  
Third, the policy presented in this paper considers an arbitrary partition that satisfy some cardinality constraints. It is of interest to come up with \emph{smarter} partitioning policies that take into consideration the travel times, and the difficulty of detection at each region. 
Last, to construct the fastest mixing Markov chain with desired
stationary distribution, we relied on time-homogeneous Markov chains. A
time varying Markov chain may achieve a faster convergence to the desired
stationary distribution~\cite{JG-JB:05}.  This is also an interesting
direction to be pursued.

\section*{Appendix}

\renewcommand{\theequation}{A-\arabic{equation}}
  % redefine the command that creates the equation no.
\setcounter{equation}{0}  % reset counter 
\subsection{Probabilistic guarantee to the uniqueness of critical
  point}
We now provide probabilistic guarantee for
Conjecture~\ref{conj:unique}. The average detection delay for a single
vehicle under a stationary policy $\q$ is
\[
\gav(\q) =\Big(\sum_{i=1}^n \frac{v_i}{q_i}\Big) \Big(\sum_{i=1}^n q_i \Tbar _i + \sum_{i=1}^n \sum_{j=1}^n q_i q_j d_{ij}\Big),
\] 
where $v_i = w_i \etab/\dist_i$ for each $i\in \until{n}$. A local
minimum of $\gav$ can be can be found by substituting
$q_{n}=1-\sum_{j=1}^{n-1} q_j$, and then running the gradient descent
algorithm from some initial point $\q_0\in \Delta_{n-1}$ on the
resulting objective function.

Let $\boldsymbol{v}=(v_1,\ldots,v_n)$ and
$\boldsymbol{T}=(\Tbar_1,\ldots,\Tbar_n)$.  We assume that the
parameters $\{\boldsymbol{v}, \boldsymbol{T}, D, n\}$ in a given
instance of optimization
problem~\eqref{eq:original-objective-function} and the chosen initial
point $\q_0$ are realizations of random variables sampled from some
space $\mc K$.  For a given realization $\kappa\in\mc K$, let the
realized value of the parameters be $\{\boldsymbol{v}(\kappa),
\boldsymbol{T}(\kappa), D(\kappa), n(\kappa)\}$, and the chosen
initial point be $\q_0(\kappa)$. The associated optimization problem
is:
\begin{equation}\label{eq:random-problem}
\underset{\q \in \Delta_{n(\kappa)-1}}{\minimize} \quad \gav (\q  \,|\, \kappa),
\end{equation}
where, for a given realization $\kappa \in \mc K$, $\map{\gav (\cdot
  \,|\, \kappa)}{\Delta_{n(\kappa)-1}}{\real_{>0}\union \{+\infty\}}$
is defined by
\begin{multline*}\label{eq:sampled-optimization-problem}
\!\!\!\!\gav(\q \,|\, \kappa) = \Big(\sum_{i=1}^{n(\kappa)} \frac{v_i(\kappa)}{q_i}\Big) \Big(\!\sum_{i=1}^{n(\kappa)} q_i \Tbar _i(\kappa) + \sum_{i=1}^{n(\kappa)} \sum_{j=1}^{n(\kappa)} q_i q_j d_{ij}(\kappa)\!\Big).
\end{multline*}
For a given realization $\kappa$, let $\map{gd(\cdot \,|\,
  \kappa)}{\Delta_{n(\kappa)-1}}{\Delta_{n(\kappa)-1}}$ be the
function that determines the outcome of the gradient descent algorithm
applied to the function obtained by substituting
$q_{n(\kappa)}=1-\sum_{j=1}^{n(\kappa)-1} q_j$ in $\gav(\q \,|\,
\kappa)$. In other words, the gradient descent algorithm starting from
point $q_{0}(\kappa)$ converges to the point $gd(q_{0}(\kappa) \,|\,
\kappa)$.  Consider $N_1$ realizations $\{\kappa_1,\ldots,
\kappa_{N_1}\}\in \mc K^{N_1}$.  Let $\subscr{\q}{optimal}(\kappa)=
gd(\frac{1}{n(\kappa)}\boldsymbol{1}_{n(\kappa)} \,|\, \kappa) $, and
define
\[
\hat{\gamma} =\max\setdef{\|gd(\q_{0}(\kappa_s) \,|\, \kappa_{s})-\subscr{\q}{optimal}(\kappa_s)\|}{s\in\until{N_1}}.
\]
%
%
%
%
%For a given realization $\kappa$, a local minimum of the optimization
%problem~\eqref{eq:random-problem} can be found by substituting
%$q_{n(\kappa)}=1-\sum_{j=1}^{n(\kappa)-1} q_j$, and then running the
%gradient descent algorithm from some initial point $\q_0\in
%\Delta_{n(\kappa)-1}$ on the resulting objective function.  For a
%given realization $\kappa$, we sample $N_1$ initial points
%$\q_{0s}(\kappa), s \in\until{N_1}$ in the simplex
%$\Delta_{n(\kappa)-1}$, and run the gradient descent algorithm from
%each initial point to solve the optimization
%problem~\eqref{eq:random-problem}. 
 
It is known~\cite{GCC-FD-RT:11} that if $N_1\ge - (\log \nu_1)
/\mu_1$, for some $\mu_1, \nu_1 \in {]0,1[}$, then, with at least
confidence $1-\nu_1$, it holds
\begin{multline*}
  \mathbb{P}(\setdef{\q_0(\kappa)\!\in
    \!\Delta_{n(\kappa)\!-\!1}}{\|gd(\q_{0}(\kappa) \,|\, \kappa)-\subscr{\q}{optimal}(\kappa)\| \le\hat{\gamma}} )\\
  \ge 1-\mu_1,
\end{multline*}
for any realization $\kappa\in \mc K$.

%Now consider the realizations $\setdef{\kappa_u\in\mc K}{u
%  \in\until{N_2}}$, and define
%\[
%\hat{\gamma}=\max\setdef{\hat{\gamma}(\kappa_u)}{u\in\until{N_2}}.
%\]
%If $N_2\ge - (\log \nu_2) /\mu_2$, for some $\mu_2, \nu_2 \in {]0,1[}$,
%then with at least probability $1-\nu_2$
%\[
%\mathbb{P}(\setdef{\kappa\in\mc K}{\hat{\gamma}(\kappa) \le \hat{\gamma}}) \ge 1-\mu_2.
%\]

We sample the following quantities: the value $n$ as uniformly
distributed in $\{3,\ldots,12\}$; each coordinate of the $n$ regions
in two dimensional space from the normal distribution with mean $0$
and variance $100 $; the value $T_i$, for each $i\in\until{n}$, from
the half normal distribution with mean $0$ and variance $100$; and the
value $v_i$, for each $i\in \until{n}$, uniformly from ${]0,1[}$.  For
a realized value of $n$, we chose $\q_0$ uniformly in $\Delta_{n-1}$.
Let the matrix $D$ be the Euclidean distance matrix between the $n$
sampled regions.

We considered $N_1=1000$ realizations of the parameters
$\{\boldsymbol{v}, \boldsymbol{T}, D, n\}$ and initial value $\q_0$.
The sample sizes were determined for $\mu_1=0.01$ and $\nu_1=10^{-4}$.
The value of $\hat{\gamma}$ obtained was $10^{-4}$.  Consequently, the
gradient descent algorithm for the optimization
problem~\eqref{eq:original-objective-function} starting from any
feasible point yields the same solution with high probability. In
other words, with at least confidence level $99.99\%$ and probability
at least $99\%$, the optimization
problem~\eqref{eq:original-objective-function} has a unique critical
point at which the minimum is achieved.

\footnotesize

\bibliographystyle{plainnat}
\bibliography{}
\end{document}